\documentclass[nohyperref]{article}

\usepackage{microtype}
\usepackage{graphicx}
\usepackage{subfigure}
\usepackage{booktabs} 
\graphicspath{{../}}
\usepackage[utf8]{inputenc} 
\usepackage[T1]{fontenc}    
\usepackage{hyperref}       
\usepackage{url}            
\usepackage{booktabs}       
\usepackage{amsgen,amsmath,amstext,amsbsy,amsopn,amssymb,mathabx,mathrsfs,amsthm,bm,bbm,enumitem,xr,multirow,fancyvrb}
\usepackage{mathtools}
\usepackage{nicefrac}       
\usepackage{microtype}      
\usepackage{xcolor}         
\usepackage{multirow}
\usepackage{graphics}
\usepackage{graphicx}
\usepackage{float}
\usepackage{parskip}

\usepackage{hyperref}

\DeclareMathOperator*{\argmax}{arg\,max}
\DeclareMathOperator*{\argmin}{arg\,min}

\newcommand{\KL}{\mathrm{KL}}
\newcommand{\ML}{\mathrm{ML}}
\newcommand{\pess}{\mathrm{PASTA}}
\newcommand{\cO}{\mathcal{O}}

\newcommand{\cD}{\mathcal{D}}
\newcommand{\cR}{\mathcal{R}}
\newcommand{\cV}{\mathcal{V}}

\newcommand{\cF}{\mathcal{F}}
\newcommand{\cN}{\mathcal{N}}

\newcommand{\bx}{\bm{x}}
\newcommand{\bbE}{\mathbb{E}}
\newcommand{\bbR}{\mathbb{R}}
\newcommand{\bbS}{\mathbb{S}}
\newcommand{\bbP}{\mathbb{P}}
\newcommand{\bbQ}{\mathbb{Q}}

\newcommand{\rd}{\mathrm{d}}
\newcommand{\sfL}{\mathsf{L}}

\usepackage[accepted]{icml2023}

\usepackage{amsmath}
\usepackage{amssymb}
\usepackage{mathtools}
\usepackage{amsthm}

\usepackage[capitalize,noabbrev]{cleveref}

\theoremstyle{plain}
\newtheorem{theorem}{Theorem}[section]

\newtheorem{lemma}[theorem]{Lemma}

\theoremstyle{definition}
\newtheorem{definition}{Definition}[section]
\newtheorem{assumption}{Assumption}[section]
\theoremstyle{remark}

\newcommand\independent{\protect\mathpalette{\protect\independenT}{\perp}}
\def\independenT#1#2{\mathrel{\rlap{$#1#2$}\mkern2mu{#1#2}}}	

\usepackage[textsize=tiny]{todonotes}

\icmltitlerunning{Pessimistic Assortment Optimization}

\begin{document}

\onecolumn{
\icmltitle{PASTA: Pessimistic Assortment Optimization}



\icmlsetsymbol{equal}{*}

\begin{icmlauthorlist}
\icmlauthor{Juncheng Dong}{equal,jd}
\icmlauthor{Weibin Mo}{equal,yyy}
\icmlauthor{Zhengling Qi}{comp}
\icmlauthor{Cong Shi}{sch}
\icmlauthor{Ethan X. Fang}{aura}
\icmlauthor{Vahid Tarokh}{jd}
\end{icmlauthorlist}

\icmlaffiliation{jd}{Department of Electrical and Computer Engineering, Duke University, Durham, NC 27705, United States.}
\icmlaffiliation{yyy}{Krannert School of Management, Purdue University, West Lafayette, Indiana}
\icmlaffiliation{comp}{Department of Decision Sciences, George Washington University, Washington D.C.}
\icmlaffiliation{sch}{Department of Industrial \& Operations Engineering
University of Michigan, Ann Arbor, Michigan}
\icmlaffiliation{aura}{Department of Department of Biostatistics and Bioinformatics, Duke University, Durham, NC 27705, United States.}

\icmlcorrespondingauthor{Zhengling Qi}{qizhengling@gwu.edu}


\vskip 0.3in
}



\printAffiliationsAndNotice{\icmlEqualContribution} 

\begin{abstract}
We consider a class of assortment optimization problems in an \emph{offline} data-driven setting. A firm does not know 
the underlying customer choice model but has access to an offline dataset consisting of the historically offered assortment set, customer choice, and revenue. The objective is to use the offline dataset to find an optimal assortment. Due to the combinatorial nature of assortment optimization, the problem of insufficient data coverage is likely to occur in the offline dataset. Therefore, designing a provably efficient offline learning algorithm becomes a significant challenge. To this end,  we propose an algorithm  referred to as \emph{Pessimistic ASsortment opTimizAtion} (PASTA for short) designed based on the principle of pessimism, that can correctly identify the optimal assortment by only requiring the offline data to cover the optimal assortment under general settings. In particular, we establish a regret bound for the offline assortment optimization problem under the celebrated multinomial logit model. We also propose an efficient computational procedure to solve our pessimistic assortment optimization problem. Numerical studies demonstrate the superiority of the proposed method over the existing baseline method. 
\end{abstract}

\section{Introduction}\label{sec:intro}
One of the most critical problems faced by a seller is to select products for presentation to potential buyers. Often faced with limited display spaces and storage costs in both brick-and-mortar and online retailing, the seller needs to carefully choose a set of products from the vast collection of all available products for displaying to its customers. In this line, the problem of selecting an \textit{assortment}, i.e., a collection of products from all available products, in order to maximize the seller's revenue is the \textit{assortment optimization problem}. Obviously, the choice behavior of customers \citep{mcfadden1981econometric} is of great importance in the problem of assortment optimization. Without loss of generality, we assume the choice of each customer can be described by a preference vector $\theta^*$. This subsumes the seminal multinomial logit (MNL) model \citep{mcfadden1973conditional} which is arguably the most well-studied and widespread models in assortment optimization literature (Please see Section~\ref{sec:mnl} for more details)~\citep{AO_MNL,DAO,AO_card, AO_LP,DAO_nest_MNL,aouad2022ast_mnl}. 

In practice, $\theta^*$ is often unknown and needs to be estimated. Assuming no historical data of customers, \textit{dynamic assortment optimization} adaptively learns $\theta^*$ in a trial-and-error fashion by updating the assortment and observing the subsequent choices of customers sequentially~\citep{DAO,chen2020dynamic,rusmevichientong2020dynamic,chen2021optimal_dynamic,li2022online}. Meanwhile, in our era of Big Data, companies often collect abundant customer data. Therefore, it is often in companies' best interest to learn from the existing (potentially massive) offline datasets rather than starting from scratch. Moreover, offline learning is beneficial since online exploration can sometimes be expensive or infeasible. Hence, we take the first stab to formally study the following important question faced by every seller.

{\bf Research Question:} \emph{Given a pre-collected offline dataset of historically offered assortment, customers choices, and revenue, how can we find an efficient and theoretically justified offline algorithm to estimate the optimal assortment set without unrealistic assumptions on the offline dataset?}

When the dataset is not adaptively collected, it is not uncommon to encounter the challenge of insufficient coverage of data. For estimators such as the maximum likelihood estimator (MLE) to approximate $\theta^*$ accurately, the offline dataset must include sufficiently many assortments and customer choices. In other words, the data-collecting process needs to sufficiently explore different assortments (by the seller) and different choices (by the customers). This is unlikely to happen for offline datasets because the seller would not choose unreasonable assortments whose expected revenues are obviously suboptimal, and the customers would not choose products against their preferences. 

{\bf Major Contributions.} The main contribution of this work is two-fold. First, based on the principle of pessimism, we propose the \textit{Pessimistic ASsortment opTimizAtion} (PASTA for short) framework, which correctly identifies the optimal assortment. In particular, our framework only requires that the offline dataset covers the optimal assortment set instead of all possible (combinatorially many) assortment sets. Second, we derive the first finite-sample regret bound for offline assortment optimization under the multinomial logit (MNL) model (Please see Section~\ref{sec:mnl}), one of the most widely used models for modeling customers' choices. We subsequently propose an algorithm, also with the name \textit{PASTA}, that can efficiently solve the pessimistic assortment optimization problems. Experiments on the simulated datasets (so that $\theta^*$ is known) corroborate the efficacy of pessimistic assortment optimization. 

{\bf Paper Organization.} We briefly review the related work in Section~\ref{sec:related_work} and the preliminaries in Section~\ref{sec:pre}. We propose the pessimistic assortment optimization in Section~\ref{sec:pess}. In Section~\ref{sec:theory}, we present the theoretical results. In Section~\ref{sec:mnl}, we study pessimistic assortment optimization under the MNL model as a concrete example. In Section~\ref{sec:algo}, we propose an algorithm that can solve the problem efficiently. We provide experimental results in Section~\ref{sec:exp}, after which we conclude. 

\section{Related Work}\label{sec:related_work}
\noindent{\bf Assortment Optimization.} 
The assortment optimization problem under the MNL model without any constraints was first studied in \cite{AO_MNL}. Then more complicated assortment optimization problems under various types of constraints, including space requirement~\citep{AO_space} and cardinality~\citep{AO_card}, were considered. \cite{AO_LP} proposed a linear programming (LP) formulation of the assortment optimization problem that includes several previous works as special cases corresponding to different constraints in the formulation of LP. This line of work assumes that the true parameters of the customer models are known (or at least can be accurately estimated)~\citep{ass_opt_2014,ass_opt_2015,ass_opt_2019,ass_opt_2020,liu2020ass_opt,ass_opt_2021}. Another closely related line of work is \textit{dynamic assortment optimization}~\citep{DAO}. In the setting of dynamic assortment optimization, the seller without any prior information about the customers, has finite selling horizons in which it observes the choices of customers and, based on the observed behaviors, optimize their assortments in an adaptive, trial-and-error fashion~\citep{DAO_Card,DAO_near_optimal,DAO_nest_MNL,rusmevichientong2020dynamic,chen2020dynamic,chen2021optimal_dynamic}. In comparison with the online setting used in dynamic assortment optimization, our work departs from the existing literature by focusing on the offline setting where the seller only has collected datasets but not any control on the data-collecting process. 

\noindent{\bf Pessimism in Offline Learning.} The principle of pessimism has been successfully used in reinforcement learning (RL) for finding an optimal policy with pre-collected datasets. On the empirical side, it has helped with improving the performance of both the model-based approach and value-based approach in offline setting~\citep[e.g.,][]{model_based_RL_pess_2020,model_based_RL_pess_2_2020,value_based_RL_pess_2020}. The importance of pessimism has been analyzed and verified theoretically in the setting of RL~\citep{theory_RL_pess_3_2021,fu2022offline}. Our work main contribution is to take a pessimistic approach to assortment optimization problems and demonstrate its empirical and theoretical values. Moreover, our work differs from the above works by focusing on a  decision-making problem with exponentially many choices.

\section{Preliminary}\label{sec:pre}
Let $[N]\doteq\{1,2,..,N\}$ denote the set of $N$ distinct items. For each item $i$, a feature vector $x_{i} \in \mathbb{R}^d$ is available. Assume that $\{x_{i}\}_{i \in [N]}$ are fixed vectors.
Denote the collection of all possible assortments under \textit{consideration} by $\mathbb{S} \subseteq 2^{[N]}\setminus\{\emptyset\}$. 
For the offline data, we define a random vector $(S, A, R)$ from each customer, where $S \subseteq [N]$ denotes an assortment presented to the customer, $A \in S \cup \{0\}$ denotes the item purchased by the customer for $A \in S$ ($A = 0$ where no purchase is made), and $R$ denotes the corresponding revenue. The ultimate goal of assortment optimization is to find an optimal set of items $s^\ast \in \bbS$ for all customers to maximize the expected revenue. A specific goal of this work is to study how to leverage the offline data, which consists of i.i.d. samples of the random triplet $(S, A, R)$ in order to learn an optimal assortment.

For the assortment optimization with offline data, a fundamental question is to estimate the expected revenue for an unexplored assortment $s \in \bbS$. This amounts to addressing the causal relationship between assortment and revenue. 
Under the celebrated potential outcome framework \citep{rubin1974estimating}, let the random variable $R(s)$ be the potential revenue 
under an intervention that the assortment is set to be $s \in \mathbb{S}$.
Our goal is to find an optimal assortment 
\[s^\ast \in \argmax_{s \in \mathbb{S}}\, \mathbb{E}[R(s)].\]
Note that the expected potential revenue $\mathbb{E}[R(s)]$ defined in the counterfactual world may not be identifiable from the observed data without additional assumptions. Throughout this paper, we make the following standard consistency and un-confoundedness assumptions in causal inference.

\begin{assumption}\label{ass: consistency}[{\sc Consistency}] 
With probability one, the observed revenue coincides with the potential revenue of the observed assortment. That is, $R = R(S)$ almost surely.
\end{assumption}
\vspace{-0.3cm}
\begin{assumption}\label{ass: unconfoundedness}[{\sc Un-confoundedness}] The potential revenues are independent variables of the observed assortment, i.e., $\{R(s)\}_{s \in \bbS} \independent S$.
\end{assumption}
Assumption \ref{ass: consistency} ensures that the observed revenue is consistent with the potential revenue of purchasing item $A$ $(\neq 0)$, or no purchase if $A = 0$, under the observed assortment $S$. Assumption \ref{ass: unconfoundedness} rules out possible unobserved factors that could confound the causal effect of assortment on revenue\footnote{
    With the observed features $\{x_{j}\}_{j \in [N]}$, it can be possible to relax Assumption \ref{ass: unconfoundedness} to a more plausible condition: the independence holds conditional on the observed features. However, for notation simplicity, without loss of generality, we consider Assumption \ref{ass: unconfoundedness}.
}.

Denote $ \pi_{S}(s) \doteq \bbP(S=s) $ as the probability of observing assortment $s$ in the offline data. To non-parametrically identify $\bbE[R(s)]$ for every $s \in \mathbb{S}$,  we  further require the following positivity assumption \citep{imbens2015causal}. 

\begin{assumption}\label{ass: positivity}[{\sc Positivity Everywhere}]
For all $s \in \bbS$,  the probability $\pi(s)$ of observing assortment $s$ is positive (i.e. $\pi(s) > 0$).
\end{assumption}
Assumption \ref{ass: positivity} requires that every assortment can  be observed with a positive chance in the offline data. \emph{This is a strong assumption that will be later relaxed it Assumption \ref{ass: technical assumptions} (I), i.e., requiring positivity only at optimum.} With Assumptions \ref{ass: consistency}--\ref{ass: positivity}, we can identify the effect of an assortment set via inverse propensity score weighting \citep{rosenbaum1983central}: for any $s \in \mathbb{S}$,
\begin{align}
    \bbE[R(s)] =  \bbE\left\{\frac{\mathbb I(S = s)R}{\pi_S(s)}\right\},
    \label{eq:IPW}
\end{align}
where the expectation in the right-hand-side is taken with respect to the data distribution of $(S, A, R)$. However, when the number of possible assortments $|\bbS|$ grows exponentially in $N$, 
Assumption \ref{ass: positivity} rarely holds for all $s \in \bbS$ in practice, given potentially limited offline data particularly when $N$ is large. Moreover, when an assortment $s$ corresponds to an inferior expected revenue $\bbE[R(s)]$, it may not be considered by the seller at all. As a consequence, the probability of observing such an assortment $\pi_{S}(s)$ is zero. These may prevent us from estimating \eqref{eq:IPW} for every assortment $s \in \bbS$. 

We may tempt to use the following identification strategy:
\begin{align}
        &~ \bbE[R(s)] = \bbE(R|S=s) \quad (\text{by Assumptions \ref{ass: consistency}-\ref{ass: unconfoundedness}}) \nonumber \\
        &\ = \bbE[\bbE(R|S=s,A)] = \sum_{i \in s \cup \{0\}} \pi_{A}(i|s;\bx)r_{s,i},
    \label{eq:id}
\end{align}
where $\bx \doteq \{x_{j}\}_{j \in [N]}$ are the features across items, $\pi_{A}(i|s;\bx) \doteq \bbP(A=i|S=s)$ is the customer's choice probability \citep{mcfadden1973conditional} of purchasing the $i$-th item given an assortment $s$,  $r_{s,i} \doteq \bbE(R|S=s,A=i)$ is the conditional expected revenue given the assortment $s$ with the $i$-th item being purchased. For ease of notation, we omit the features $\bx$ in $\pi_{A}$ when there is no confusion.
Identifying $\bbE[R(s)]$ as above requires the knowledge of $\pi_{A}(i|s)$ and $r_{s,i}$, which can be learned from data. Although such an identification approach does not explicitly depend on $\pi_{S}(s)$, full identification of $\pi_{A}(i|s)$ and $r_{s,i}$ requires the positivity of $\pi_{S}(s)$ for every $s \in \mathbb S$ as assumed above.

Despite the aforementioned challenge of insufficient coverage over assortments, we argue that finding an optimal assortment $s^{\ast}$ may not necessarily require $\pi_{S}(s) > 0$ everywhere but only at the optimal assortment $s^\ast$. In particular, based on \eqref{eq:id}, when computing
\begin{align}
    s^{\ast} \in \argmax_{s \in \bbS}\sum_{i \in s \cup \{0\}}\pi_{A}(i|s)r_{s,i}, 
    \label{eqn: optimality condition}
\end{align}
we may not necessarily need to estimate $\pi_{A}(i|s)$ and $r_{s,i}$ well for $s \neq s^\ast$, as long as \emph{sub-optimal assortments can be safely ruled out during the optimization}. Our insight is that the estimation of $\pi_{A}(i|s)$ and $r_{s,i}$ for the less seen assortment $s$ in the data often incurs large errors. Deploying pessimism by taking the estimation error into consideration can rule out those assortments \cite{theory_RL_pess_3_2021}, while standard predict-then-optimize \citep{bertsimas2020predictive} or empirical maximization approaches \citep{zhao2012estimating} may suffer from an overestimation of $\mathbb E[R(s)]$. Hence, in our proposed pessimistic assortment optimization framework, we only require the positivity at optimum $\pi_S(s^\ast) > 0$, which is a much weaker assumption than that of Assumption \ref{ass: positivity}.

In this paper, we focus on handling the estimation error from $\pi_{A}(i|s)$ while assuming that $r_{s,i}$ is known. This is a typical assumption in the literature of assortment optimization~\citep{AO_MNL,AO_LP,ass_opt_2019,ass_opt_2021}. Our framework can be naturally extended to the scenario where we need to estimate $r_{s,i}$'s. For optimization tractability, we further assume that $r_{s,i} = r_{i}$ that the expected revenue depends only on the purchased item but not on the underlying assortment. This assumption is reasonable in many applications where the revenue is a deterministic consequence of a purchased item. This can also be easily extended under our pessimism framework but could result in a more complicated assortment optimization problem.

Below, without loss of generality, we assume that $ r_{i} \ge 0 $ for $i \in [N] $, while $r_{0} = 0$ (no purchase incurs zero revenue). For any vector $x$, let $x^\top$  and $||x||_2$ respectively denote the transpose and $\ell_2$-norm of $x$. For any set $A$, let $|A|$ denote the cardinal number of $A$. For any two sequences $\{\varpi(n)\}_{n \geq 1}$ and $\{\gamma(n)\}_{n \geq 1}$, we write $\varpi(n) \gtrsim  \gamma(n)$ (resp. $\varpi(N) \lesssim \gamma(n)$) whenever there exist constants $c_1 > 0$
(respectively $c_2 > 0$ such that $\varpi(n) \geq c_1 \gamma(n)$ (resp. $\varpi(n) \leq c_2 \gamma(n)$). Moreover, we write $\varpi(n) \eqsim  \gamma(n)$ whenever $\varpi(n) \gtrsim  \gamma(n)$ and $\varpi(n) \lesssim  \gamma(n)$.

\section{Pessimistic Offline Assortment Optimization}\label{sec:pess}
In this section, we introduce our pessimistic offline assortment optimization framework. To this end, based on Eq.~\eqref{eq:id}, we first estimate the choice probability $\pi_A(i|s)$ from offline data. Subsequently we calculate optimizing values in optimization problem ~\eqref{eqn: optimality condition} using a plug-in estimator of $\pi_{A}(\cdot)$. 
Consider a generic form of model $\pi_{A}(a|s;\theta^{\ast},\bx)$ with the unknown true  parameter~$\theta^{\ast}$. Again, for ease of notation, we omit the features $\bx$ in $\pi_{A}$ when there is no confusion. We remark that $\theta^{\ast}$ could be either finite-dimensional or infinite-dimensional. 
Given an offline dataset $\cD = \{S_i, A_i, R_{i}\}_{i=1}^{n}$, where $n$ is the sample size, one can estimate the model parameter $\theta^\ast$ via maximum likelihood estimator (MLE). Specifically, define the likelihood-based loss function $\widehat L_n(\theta)$ as
\[ \widehat L_n(\theta) = -\frac{1}{n}\sum_{i=1}^{n}\log\pi_{A}(A_i|S_i;\theta). \]
Then the MLE of the unknown parameter $\theta^*$ is
$ \widehat{\theta}_{\ML,n} \in \argmin_{\theta \in \Theta}\widehat{L}_{n}(\theta) $,
where $\Theta$ is a pre-specified parameter space. Let
\begin{equation*}
\cV(s;\widehat{\theta}_{\ML,n}) \doteq \sum_{i\in s} \pi_A(i|s; \widehat{\theta}_{\ML,n})r_{i}.
\end{equation*}
Here, we define $\cV(s;\theta^\ast)$ as the \textit{value function} of $s$ with the customer choice model for $\pi_{A}$ depending on the parameter $\theta^\ast$. The plug-in estimator of the optimal assortment based on \eqref{eqn: optimality condition} is 
\[ \widehat s_{\ML,n} \in \argmax_{s \in \bbS}\left\{\cV(s;\widehat{\theta}_{\ML,n})\right\}. \]
 The MLE-based approach first plugs in the MLE of $\theta^\ast$, and then directly optimizes the corresponding estimated value function. 



As discussed before, a disadvantage of the above estimate-then-optimize approach is that the estimation error of $\widehat{\theta}_{\ML,n}$ caused by insufficient data coverage may result in the overestimation of $\cV(s;\theta^\ast)$, which will propagate to downstream optimization. Alternatively, we can quantify the estimation uncertainty by considering the following likelihood-ratio-test-based confidence region \citep{owen1990empirical}:
\[ \Omega_{n}(\alpha_{n}) \doteq \{ \theta \in \Theta: \widehat{L}_{n}(\theta) - \widehat{L}_{n}(\widehat{\theta}_{\ML,n}) \leq \alpha_{n}\}, \]
where $\alpha_{n} >0$ is pre-specified. Later we analyze the MNL model as a special case (Please see Section~\ref{sec:mnl}). With $\alpha_{n}$ chosen as $\cO(d/n)$, we establish in Theorem \ref{thm:empirical_likelihood_bound} that $\theta^{*} \in \Omega_{n}(\alpha_{n})$ with high probability. Such a guarantee does not require any data coverage assumption on assortments.

For now on, for simplicity, we drop $\alpha_{n}$ and write $\Omega_{n}$ for $\Omega_{n}(\alpha_{n})$ when there is no ambiguity.

In order to robustify assortment optimization against plug-in estimation errors, we consider a pessimistic version of \eqref{eqn: optimality condition} by taking the estimation uncertainty from $\Omega_{n}$ into account.
Specifically, we propose the \textbf{Pessimistic ASsortment opTimizAtion (PASTA)} by solving
\begin{equation}\label{eqn: pessimistic algorithm}
    \widehat s_{\pess,n} \in \argmax_{s \in \bbS} \min_{\theta \in \Omega_n} \mathcal{V}(s;\theta).
\end{equation}
Here, for a fixed assortment $s \in \bbS$, the inner layer of minimization computes the worst-case value among all possible model parameters $\theta$ within the confidence set $\Omega_{n}$. In particular, if the estimated value $\cV(s;\widehat{\theta}_{\ML,n})$ for $s$ is highly uncertain due to insufficient data coverage,  the worst-case value $\min_{\theta \in \Omega_{n}}\cV(s;\theta)$ is likely much smaller than $\cV(s;\theta^\ast)$. In that case, the outer layer of \eqref{eqn: pessimistic algorithm} may prefer another assortment with a relatively higher worst-case value.
In this way, the inner layer of \eqref{eqn: pessimistic algorithm} rules out those assortments with less frequency in the offline data. Hence, one essential advantage of such a strategy is that it avoids an overestimation of the value function. In other words, by the plug-in approach, with a non-negligible chance, the estimated value $\cV(s;\widehat{\theta}_{\ML,n})$ can be much larger than the truth $\cV(s;\theta^*)$, which  further leads to a possibly sub-optimal assortment but optimized by the MLE-based approach. In contrast, PASTA is aware of insufficient data coverage, and hence more pessimistic about those highly uncertain value estimates. In the next section, we theoretically analyze the advantage of the PASTA approach.


\section{Theoretical Results}\label{sec:theory}
In this section, we show that the PASTA method \eqref{eqn: pessimistic algorithm} enjoys a generic regret guarantee under a weak assumption of \emph{positivity at optimum} that is $\pi_{S}(s^\ast) > 0$. Specifically, given $\widehat s_{\pess,n}$ in~\eqref{eqn: pessimistic algorithm}, we adopt the following regret as the performance metric to evaluate the PASTA's performance
\[ \cR(\widehat{s}_{\pess,n}) = \cV(s^{*};\theta^{*}) - \cV(\widehat{s}_{\pess,n};\theta^{*}). \] 
We aim to derive a regret bound for $\cR(\widehat{s}_{\pess,n})$ under generic conditions. Denote $L(\theta) = \bbE[-\log\pi_{A}(A|S;\theta)]$ as the population loss function. All detailed proofs can be found in the Appendix~\ref{sec:appendix}.

We first show that whenever $\theta^{*} \in \Omega_{n}$ that the confidence region covers the true parameter, the regret of the PASTA method can be calibrated by the worst-case estimation error among $\theta \in \Omega_{n}$ of the value function at the optimal assortment $s^\ast$.

\begin{lemma} \label{lemma:general_pasta_regret}
Let $\widehat{s}_{\pess,n}$ be the solution by the PASTA method defined in \eqref{eqn: pessimistic algorithm}. 
If $\theta^* \in \Omega_n$, then
\[ \cR(\widehat{s}_{\pess,n}) \le \max_{\theta \in \Omega_n} \big\{\cV(s^{*};\theta^{*}) -  \cV(s^{*};\theta)\big\}. \]
\end{lemma}
{\it Proof of Lemma \ref{lemma:general_pasta_regret}.}
\begin{align*}
&~ \cV(s^*;\theta^*)-\cV(\widehat{s}_{\pess,n};\theta^*) \\
& \ \leq \cV(s^*;\theta^*) - \min_{\theta \in \Omega_n}\cV(\widehat{s}_{\pess,n};\theta)
&& \hspace{0.8cm} (\text{by} ~ \theta^\ast \in \Omega_{n})\\
 &\ \leq \cV(s^*;\theta^*) - \min_{\theta \in \Omega_n}\cV(s^*;\theta) 
&& \hspace{-0.8cm} (\text{by} ~ \widehat{s}_{\pess,n} ~ \text{solves \eqref{eqn: pessimistic algorithm}})\\
&\  = \max_{\theta \in \Omega_n} \big\{\cV(s^{*};\theta^{*}) -  \cV(s^{*};\theta)\big\}. 
&&\hspace{2.4cm} \qedsymbol{}
\end{align*} 
Lemma \ref{lemma:general_pasta_regret} highlights the benefit of pessimistic optimization. In particular, guarantees for the regret of estimate-then-optimize approaches require the control of estimation error uniformly over all $ s \in \bbS $, that is, $\sup_{s \in \bbS}\big|\cV(s;\widehat{\theta}_{\ML,n}) - \cV(s;\theta^\ast)\big| $, in order to control error caused by the greedy policy, i.e., $\big|\cV(\widehat{s}_{\ML,n},\widehat{\theta}_{\ML,n}) - \cV(\widehat{s}_{\ML,n},\theta^\ast)\big|$. This may typically entail that the value function is estimated uniformly well across $s \in \bbS$. It further explains the reason why an estimate-then-optimize approach would require the positivity Assumption \ref{ass: positivity} for all $s \in \bbS$.
In contrast, our Lemma \ref{lemma:general_pasta_regret} suggests that controlling the estimation error at $s^\ast$ can be enough. Therefore, it is sufficient for our PASTA method to require the positivity only at optimum. 

Next, we impose the following assumptions to obtain the regret guarantee of our algorithm.

\begin{assumption}\label{ass: technical assumptions}
\mbox{ }\\
(I) {\sc [Positivity at Optimum]} The probability of observing the optimal assortment is positive, that is, $\pi_S(s^\ast)>0.$ \\
(II) {\sc [Likelihood-Based Concentration]} For any $0 < \delta < 1$, with probability at least $1-\delta$, we have: (1) $\theta^\ast \in \Omega_n$, and (2)
\[ \sup_{\theta \in \Omega_n}\left|L(\theta)-L(\theta^*)-\left(\widehat{L}_{n}(\theta)- \widehat{L}_{n}(\theta^*)\right)\right|\le \alpha_n. \]
\end{assumption}

We emphasize that \emph{PASTA only requires the positivity at optimum}. Compared to the positivity at all assortments in Assumption \ref{ass: positivity}, our Assumption \ref{ass: technical assumptions} (I) is much weaker and hence more plausible to be satisfied. 
Assumption \ref{ass: technical assumptions} (II) is a generic condition for likelihood-based concentration. We later justify that (II) above indeed holds under the general MNL model in Theorem~\ref{thm:likelyhood_convergence}.
In particular,  Statement~(1) of Part~(II) requires the validity of the likelihood-ratio-test-based confidence region $\Omega_{n}$ while Statement~(2) of Part~(II) requires the concentration of the likelihood-based localized empirical process \citep{vaart1996weak}. 

The positivity at optimum is associated with a finite constant $C_{s^*}= 1/\pi_{s}(s^\ast)$ related to the learning performance. 
We also denote $r_{s^{\ast}} \doteq \max_{j \in s^\ast}r_{j}$ as the largest possible revenue among all items in $s^\ast$. Notice that both constants $C_{s^\ast}$ and $r_{s^\ast}$ depend on the optimal assortment $s^\ast$ only. In the following lemma, we establish the estimation error bound at the optimal assortment $s^\ast$.

\begin{lemma}\label{lemma:general_pasta_theta}
Under Assumption \ref{ass: technical assumptions}, for any $0 < \delta < 1$, with probability at least $1-\delta$, we have for any $\theta \in \Omega_{n} $,
\[ \cV(s^{*};\theta^{*}) -  \cV(s^{*};\theta) \lesssim \: r_{s^{\ast}} C_{s^*}\sqrt{\alpha_n}. \]
\end{lemma}

Combining Lemmas~\ref{lemma:general_pasta_regret} and \ref{lemma:general_pasta_theta}, we summarize the regret bound for PASTA in the following theorem.

\begin{theorem}\label{thm:regret}
Under Assumption \ref{ass: technical assumptions}, for any $0 < \delta < 1$, with probability $1-\delta$, we have
\[ \cR(\widehat{s}_{\pess,n}) \lesssim \: r_{s^{\ast}} C_{s^*}\sqrt{\alpha_n}. \]
\end{theorem}

\section{Application: Multinomial Logit Model}\label{sec:mnl}
In this section, we consider the Multinomial Logit Model (MNL) for customer choices  $\pi_A(a|s)$. This is one of the most widely used models in assortment optimization literature \citep{feng2022consumer}. 
Under the MNL model, we will verify Assumption \ref{ass: technical assumptions} (II) and establish the regret bound for PASTA in this case. 

Given the item-specific features $ \{ x_{i} \}_{i \in [N]} $, MNL assumes that customer's preference for the $i$-th item is proportional to $\exp(x_{i}^\top\theta^\ast)$, where $\theta^{\ast} \in \Theta$ is the underlying unknown parameter. Here, we assume that the parameter space $\Theta \subseteq \bbR^{d}$ is compact with $\theta_{\max} \doteq \sup_{\theta \in \Theta}\|\theta\|_{2} < +\infty $. 
Given an assortment $s$, the customer choice probability under MNL is given by
\begin{align}\label{eqn:mnl_ass}
    \pi_A(i|s;\theta^\ast)= \frac{\exp(x_i^\top\theta^*)}{1+\sum_{j \in s}\exp(x_j^\top\theta^*)}, \quad \forall i \in s.
\end{align}
Moreover, the probability of no-purchase is normalized to $ \pi_{A}(0|s;\theta^\ast) = 1 /(1+\sum_{j \in s}\exp(x_j^\top\theta^*)) $. Based on \eqref{eqn: optimality condition} and the MNL model \eqref{eqn:mnl_ass}, the objective function for assortment optimization can be written as
\[ \cV(s;\theta) = \frac{\sum_{i\in s}r_{i}\exp(x_{i}^\top\theta)}{1+\sum_{i \in s}\exp(x_{i}^\top\theta)}. \]




We first justify Statement~(1) of Assumption \ref{ass: technical assumptions} under the MNL model. To this end, given the compactness of $\Theta$, there exists a finite constant $C_{A}>0$ such that for all $\theta \in \Theta$, $ s \in \bbS $ and $i \in s$, we have $ 1/\pi_A(i|s;\theta) \le C_{A}$. 

\begin{lemma}\label{thm:empirical_likelihood_bound}
Consider the MNL model \eqref{eqn:mnl_ass} with a compact set $\Theta$. Assume that $\theta^\ast \in \Theta$. For any $0 < \delta < 1$,  with probability $\ge (1-\delta)$, we have
\[ \widehat{L}_{n}(\theta^*)-\widehat{L}_{n}(\widehat\theta_{\ML,n}) \lesssim \frac{C_{A}d}{n}\log\frac{\theta_{\max}}{\delta}. \]
\end{lemma}

Lemma \ref{thm:empirical_likelihood_bound} suggests that, with $\alpha_{n}$ chosen as $\frac{C_{A}d}{n}\log\frac{\theta_{\max}}{\delta}$, we can guarantee that $\theta^{*} \in \Omega_{n}$ with high probability, which justifies Statement~(1) of Assumption \ref{ass: technical assumptions} (II). In particular, the order of $\alpha_{n}$ is $\cO(d/n)$. Notice that Lemma \ref{thm:empirical_likelihood_bound} does not depend on the distribution of $S$, which implies that no data coverage assumption on the observed assortments is required. The assumption that $\theta^\ast \in \Theta$ for $\Theta$ a compact set requires that given any assortment, every product has a chance of being selected by the customer in the data. This is a mild requirement as $\theta^\ast$ is always finite.


Next, we justify Statement~(2) of Assumption \ref{ass: technical assumptions} (II) in the following theorem. 

\begin{lemma}\label{thm:likelyhood_convergence}
    Consider the MNL model \eqref{eqn:mnl_ass}. Suppose conditions in Lemma \ref{thm:empirical_likelihood_bound} hold, and $L(\theta)$ and $L_n(\theta)$ are uniformly and strongly convex. Let $\alpha_{n} \eqsim {C_{A}d \over n}\log{\theta_{\max} \over \delta}$. For $ 0 < \delta < 1 $, with probability $ \ge (1 - \delta)$, we have
    \[ \sup_{\theta \in \Omega_n}\left|L(\theta)-L(\theta^*)-\left(\widehat{L}_{n}(\theta)- \widehat{L}_{n}(\theta^*)\right)\right|\le \alpha_n. \]
\end{lemma}

Finally, with the Assumption of positivity only at optimum (Assumption \ref{ass: technical assumptions} (I)), we can apply Theorem \ref{thm:regret} to establish the regret bound for PASTA in MNL.

\begin{theorem}\label{thm:PASTA_MNL}
Consider the MNL model (given in Equation \eqref{eqn:mnl_ass}). Assume that the conditions in Lemma \ref{thm:likelyhood_convergence} hold and that $\pi_{S}(s^\ast) > 0$. Fix a $\delta \in (0, 1)$.
Suppose $\widehat{s}_{\pess,n}$ is output of PASTA with $\alpha_{n} \eqsim {C_{A}d \over n}\log{\theta_{\max} \over \delta}$. Then with probability $\ge (1-\delta)$, we have
\[ 
\cR(\widehat{s}_{\pess,n}) \lesssim \: r_{s^{\ast}} C_{s^*}\sqrt{\frac{C_{A}d}{n}\log\frac{\theta_{\max}}{\delta}}.
\]
\end{theorem}

We remark that under the MNL model (given in Equation \eqref{eqn:mnl_ass}), the order of regret is $\cO(\sqrt{d/n})$. This is due to the concentration rate of MNL's empirical likelihood ratio in Lemma \ref{thm:empirical_likelihood_bound}. Such a rate of regret bound matches  those in the literature under parametric model assumptions \cite{qian2011performance,mo2022efficient}. However, existing literature requires the positivity $\pi_{S}(s) > 0$ at every $s \in \bbS$. In contrast, Theorem \ref{thm:PASTA_MNL} only requires positivity $\pi_{S}(s^{\ast}) > 0$ at the optimal assortment $s^\ast$. Furthermore, we can show that $\min_{i \in N}\pi_A(i \mid S = [N];\theta) \leq 1/N$ for any $\theta \in \Theta$, which implies that $C_A \geq N$. Therefore, our regret is of order at least $\sqrt{N}$, where $N$ is the total number of available items. It is an interesting problem to establish the minimax lower bound of offline assortment optimization in terms of $N, n, d$ and the cardinal number of $s^\ast$. This will investigated in a subsequent work.

\section{\textit{PASTA} Algorithm}\label{sec:algo}
In this section, we propose an efficient algorithm for solving the max-min problem given in Optimization Problem \eqref{eqn: pessimistic algorithm} for the MNL model. Specifically, let
$$
\mathcal{V}(s;\theta) = \sum_{i\in s}\frac{r_i\exp(x^\top_i\theta)}{1+\sum_{j \in s}\exp(x^\top_j\theta)}
$$ and 
given the confidence set $\Omega_{n}$, we wish to solve
\begin{equation*}
    \max_{s \in \mathbb{S}} \min_{\theta \in \Omega_n} \mathcal{V}(s;\theta).
\end{equation*}
The proposed iterative algorithm is executed for a maximum of $T$ iterations. At the $t$-th iteration, given $s_t$ and $\theta_t$ from the previous iteration, we consecutively execute the following two steps:
\begin{itemize}
    \item Step 1: Compute the optimal assortment $s_{t+1}$ given $\theta_t$ (see Section~\ref{sec:step_1}).
    \item Step 2: Compute the optimal $\theta_{t+1}$ using $s_{t+1}$ (see Section~\ref{sec:step_2}).
\end{itemize}
The corresponding pseudo-code is presented in Algorithm~\ref{alg:main} below. 
\begin{algorithm}[h]
   \caption{PASTA}
   \label{alg:main}
\begin{algorithmic}
\STATE {\bfseries Input:} offline dataset $\{(S_i,A_i,R_i)\}_{i=1}^{n}$; $\alpha_n$; $\{r_i\}_{i=1}^{N}$; $\{x_i\}_{i=1}^{N}$; maximum number of iterations $T$ 
\STATE {\bfseries Output:} the solution to pessimistic assortment optimization $\widehat s$
\STATE{$\widehat{L}_{n}(\theta)\doteq -\frac{1}{n}\sum_{i=1}^{n}\log\pi_{A}(A_{i}|S_i;\theta)$}
\STATE{$\widehat\theta_{\ML,n} \gets \argmin_{\theta \in \Theta}\widehat{L}_{n}(\theta)$}
\STATE{$\Omega_n \gets \{ \theta \in \Theta: \widehat{L}_{n}(\theta) - \widehat{L}_{n}(\widehat{\theta}_{\ML,n}) \leq \alpha_{n} \} $}
\STATE{$t \gets 0; \theta_t \gets \widehat\theta_{\ML,n}$ \quad\quad/* Initialize $\theta_0$ as $\widehat\theta_{\ML,n}$ */}
\FOR{$t=1$ {\bfseries to} $T$}
\STATE $s_t \gets \textbf{SolveLP}(\theta_{t-1},\{r_i\}_{i=1}^{N},\{x_i\}_{i=1}^{N})$ 
\STATE \qquad \qquad \qquad /* Section~\ref{sec:step_1} */
\STATE $\theta_t \gets \textbf{SolveGD}(s_{t},\Omega_{n},\{r_i\}_{i=1}^{N},\{x_i\}_{i=1}^{N})$ 
\STATE \qquad \qquad \qquad /* Section~\ref{sec:step_2} */
\ENDFOR
\STATE $\hat s \gets s_T$
\end{algorithmic}
\end{algorithm}

\subsection{Optimal Assortment Computation}\label{sec:step_1}
Given the MNL model parameter $\theta_{t}$, computing the assortment $s_{t+1}$ that maximizes the expected revenue can be formulated as a linear programming (LP) problem. 

Suppose that an assortment $s$ can be represented by an $N$-dimensional binary vector $\gamma\in \{0,1\}^N$ where $\gamma_j = 1$ if and only if $j \in s$. Suppose that $s \in \bbS$ corresponds to the following feasible set for $\gamma$ with $M$ linear inequality constraints:
\[ \Gamma=\left\{\gamma \in \{0,1\}^N: \sum_{j \in N}a_{ij}\gamma_j\leq b_i ~ \text{for}~ i \in [M]\right\}, \]
where the matrix of constraint coefficients $[a_{ij}]_{i\in [M], j\in[N]}$ is a totally unimodular matrix~\cite{pang_2017}. In other words, based on the one-to-one correspondence between $s$ and $\gamma$, we have $s \in \bbS$ if and only if $\gamma \in \Gamma$. 


Next, we denote $v_i = \exp(x_{i}^\top \theta_{t})$ as the preference score for the $i$-th item. The customer choice probability under the MNL model \eqref{eqn:mnl_ass} becomes $\pi_A(i|s)=\frac{v_i}{1+\sum_{j\in s}v_j}$. The optimization for $s_{t+1}$ can be formulated as
\begin{align} \label{eqn:opt_reward}
\max_{\gamma \in \Gamma}\frac{\sum_{i\in [N]}r_iv_i\gamma_i}{1+\sum_{i\in[N]}v_i\gamma_i},
\end{align}
which is equivalent to the following linear programming problem~\cite{AO_LP}:
\begin{equation} \label{eqn:opt_reward_lp}
\begin{aligned}
\max_{w_j: j \in [N] \cup \{0\}} \ & \sum_{j\in[N]}r_j w_j\\
\text{subject to} \quad & \sum_{j\in[N]}w_j + w_0 = 1\\
& \sum_{j\in [N]}a_{ij}\frac{w_j}{v_j}\leq b_i w_0 \quad \forall i \in [M] \\
& 0 \leq \frac{w_j}{v_j} \leq w_0 \quad \forall j \in [N].
\end{aligned}
\end{equation}
In particular, we can recover the optimal solution to  Problem~\eqref{eqn:opt_reward}, denoted as $\gamma^*$, using the optimal solution to Problem~\eqref{eqn:opt_reward_lp}, denoted by $w^*$, via the following formula:
\begin{equation}\label{eqn:translate_lp_sol}
    \gamma_j^* = \frac{w_j^*}{v_j w_0^*} \quad \forall j \in [N].
\end{equation}
To conclude, at the $t$-th iteration, in order to compute an optimal assortment $s_{t+1}$ for a given $\theta_t$, we first solve an LP problem in \eqref{eqn:opt_reward_lp} for $w^\ast$. Then we recover $\gamma^*$ via \eqref{eqn:translate_lp_sol}. Finally, the updated assortment $s_{t+1}$ is obtained by the correspondence $i \in s_{t+1}$ if and only if $\gamma_{j}^* = 1$.

\subsection{Model Parameter Computation}\label{sec:step_2}
For a given optimized assortment $s_{t+1}$ from Section \ref{sec:step_1}, we aim to search for the worst-case MNL parameter $\theta_{t+1}$ from the confidence set $\Omega_{n}$ that minimizes the expected revenue. In particular, we employ a gradient descent with line search (GDLS) method to compute $\theta_{t+1}$ by solving the following problem
\begin{equation}
    \min_{\theta \in \Omega_n} \mathcal{V}(s_{t+1};\theta).
\end{equation}
Here, we remark that $\cV(s_{t+1};\theta) = {\sum_{i \in s_{t+1}}r_{i}\exp(x_{i}^\top\theta) \over 1 + \sum_{i \in s_{t+1}}\exp(x_{i}^\top\theta)}$ is a locally Lipschitz function in $\theta$. 
Given a feasible initial parameter $\theta^{(0)}\in\Omega_n$, we run at most $L$ 
gradient descent steps. Suppose $\beta_\ell$ is the step size for gradient descent in the $\ell$-th 
step. At each step $\ell = 1,2,\cdots,L$, we do a line search to maintain the feasibility. In 
particular, given $\theta^{(\ell-1)}\in\Omega_n$, we first evaluate the gradient as $\xi_\ell = \nabla_{\theta} \mathcal{V}(s_{t+1};\theta^{(\ell-1)})$.
Then we initiate $\beta_\ell$ with a pre-specified step size $\beta_\ell = \widetilde{\beta}$, 
and check whether $\theta^{(\ell)} = \theta^{(\ell-1)} - \beta_\ell \xi_\ell$ is feasible, i.e. 
$\theta^{(\ell)} \in \Omega_{n}$. If not, we set $\beta_\ell \leftarrow c\beta_\ell$ for some pre-specified $c \in (0,1)$, and recompute 
$\theta^{(\ell)} = \theta^{(\ell-1)} - \beta_\ell \xi_\ell$. Such a search is repeated until 
$\theta^{(\ell)}$ is feasible. We provide the pseudocode in Algorithm~\ref{algo:GDLS} for the overall process. Note that $L, \widetilde{\beta}, c$ are all hyper-parameters. In all of our numerical studies, we set $L=2$, $\widetilde{\beta}=0.01$ and $c=\frac{1}{2}$, which performs well empirically.


\begin{algorithm}[h]
   \caption{Gradient Descent with Line Search (GDLS)}
   \label{algo:GDLS}
\begin{algorithmic}
\STATE {\bfseries Input:} assortment $s_{t+1}$; feasible set $\Omega_{n}$;
initial parameter $\theta^{(0)}$; 
initial step size $\tilde{\beta}$; 
step shrinkage constant $c$; 
number of descent steps $L$
\STATE {\bfseries Output:} the updated parameter $\theta_{t+1}$
\STATE{$\ell \gets 0$}
\FOR{$\ell=1$ {\bfseries to} $L$}
\STATE{$\xi_\ell= \nabla_{\theta} \mathcal{V}(s_{t+1};\theta^{(\ell-1)})$ \quad/* compute the gradient */ }
\STATE{$\beta_\ell \gets \widetilde{\beta}$}
\STATE{$\theta^{(\ell)}\gets\theta^{(\ell-1)}-\beta_\ell\xi_\ell$}
\WHILE{$\theta^{(\ell)} \notin \Omega_{n}$}
\STATE{$\beta_\ell \gets c\beta_\ell$ \quad/* decrease the step size */}
\STATE{$\theta^{(\ell)}\gets\theta^{(\ell-1)}-\beta_\ell\xi_\ell$}
\ENDWHILE
\ENDFOR
\STATE{$\theta_{t+1}\gets\theta^{(\ell)}$}

\end{algorithmic}
\end{algorithm}

\section{Experiments}\label{sec:exp}
We compare the PASTA method with assortment optimization without pessimism (referred to as the \textit{baseline} method in the sequel). 
Our method and the baseline method are evaluated on synthetic data for which the optimal assortment $s^*$ and true parameter $\theta^*$ are known so that the true regrets can be computed. We describe the data generation process and the baseline method in details below.

\subsection{Data Generation}\label{sec:data_generation}
We consider the assortment optimization scenarios described by $N$, $K$, $d$, $n$ and $p$, where $N$ is the total number of available products; $K$ is the cardinality constraint of the assortments, i.e., $\mathbb S = \{s: |s| \leq K\}$; $d$ is the dimension of $\theta^*$ and $\{x_j\}_{j=1}^N$; $n$ is the sample size of the offline dataset; $p$ is the probability for sampling the optimal assortment $s^*$. Similar to \cite{chen2020dynamic}, we first generate the true preference vector $\theta^*$ as a uniformly random unit $d$-dim vector. For $i \in \{1,\dots,N\}$, we generate $r_i$ (the reward of product $i$) uniformly from the range $[0.5, 0.8]$ and generate $x_i$ (the feature of product $i$) as uniformly random unit $d$-dim vector such that $\exp(x_i^{\top}\theta^*)\leq \exp(-0.6)$ to avoid degenerate cases, where the optimal assortments include too few items. Given such information, the true optimal assortment $s^*$ can be computed. Then, we  generate an offline dataset $\mathcal{D}=\{(S_i,A_i,R_i)\}_{i=1}^{n}$ with $n$ samples. For $i \in \{1,\dots,n\}$, we generate $S_i$ following the distribution $\pi_S$ such that $\pi_S(s^*)=p$ and $\pi_S(s)=\frac{1-p}{|\mathbb{S}|-1}$, where $0<p<1$ is the probability of observing the optimal assortment $s^*$. After the assortment $S_i$ is sampled, the customer choice (action) $A_i$ is sampled according to the probability computed by MNL as in Eq.~\eqref{eqn:mnl_ass} with the true parameter $\theta^*$. 

\subsection{Baseline}
In our experiments, we use the gradient descent method to find $\widehat\theta_{\ML,n}$ that minimizes the empirical negative log-likelihood function. Then given $\widehat\theta_{\ML,n}$, the baseline method solves the assortment optimization problem by solving the linear programming problem in \eqref{eqn:opt_reward_lp}. 

\subsection{Performance Comparison}
For a given $(N,K,d,n,p)$, we repeat the data generation process in Section~\ref{sec:data_generation} to randomly  generate $50$ offline datasets. The solutions of PASTA and the baseline method are recorded in these experiments. For hyper-parameters, we set $\alpha_n = 2\widehat L_{\ML}$ where $\widehat L_{\ML} = \widehat L_n(\widehat\theta_{\ML,n})$ and the maximum of iteration $T=30$. We measure the performance with two metrics: (1) the \textit{average regret} of the solutions which indicates how far the performance of the solutions is to that of the \textit{optimal} performance (i.e., revenue of $s^*$); (2) the \textit{assortment accuracy} of the solutions (with respect to the optimal assortment $s^*$). The assortment accuracy of an assortment~$s$ is defined as the ratio of the number of correctly chosen products to the number of products in $s^*$.
The key results are summarized below. 

\textbf{Effect of Sample Size.} We set $N=40$, $K=8$, $d=16$ and $p=0.9$. We then gradually increase the number of samples $n$ . The result is presented in Figure~\ref{fig:exp_effect_n_40_60} indicating that PASTA significantly outperforms the baseline method. While the performance of the baseline method improves with increasing number of samples, the PASTA method maintains a regret that is less than $25\%$ of that of the baseline method. The same experiment repeated with an increased number of products ($N=60$, $K=15$) demonstrates that the gain of the PASTA method is stable, as presented in Figure~\ref{fig:exp_effect_n_40_60}. 

\begin{figure}[h]
\vskip 0.2in
\begin{center}
\centerline{\includegraphics[width=0.7\columnwidth, scale = 0.5]{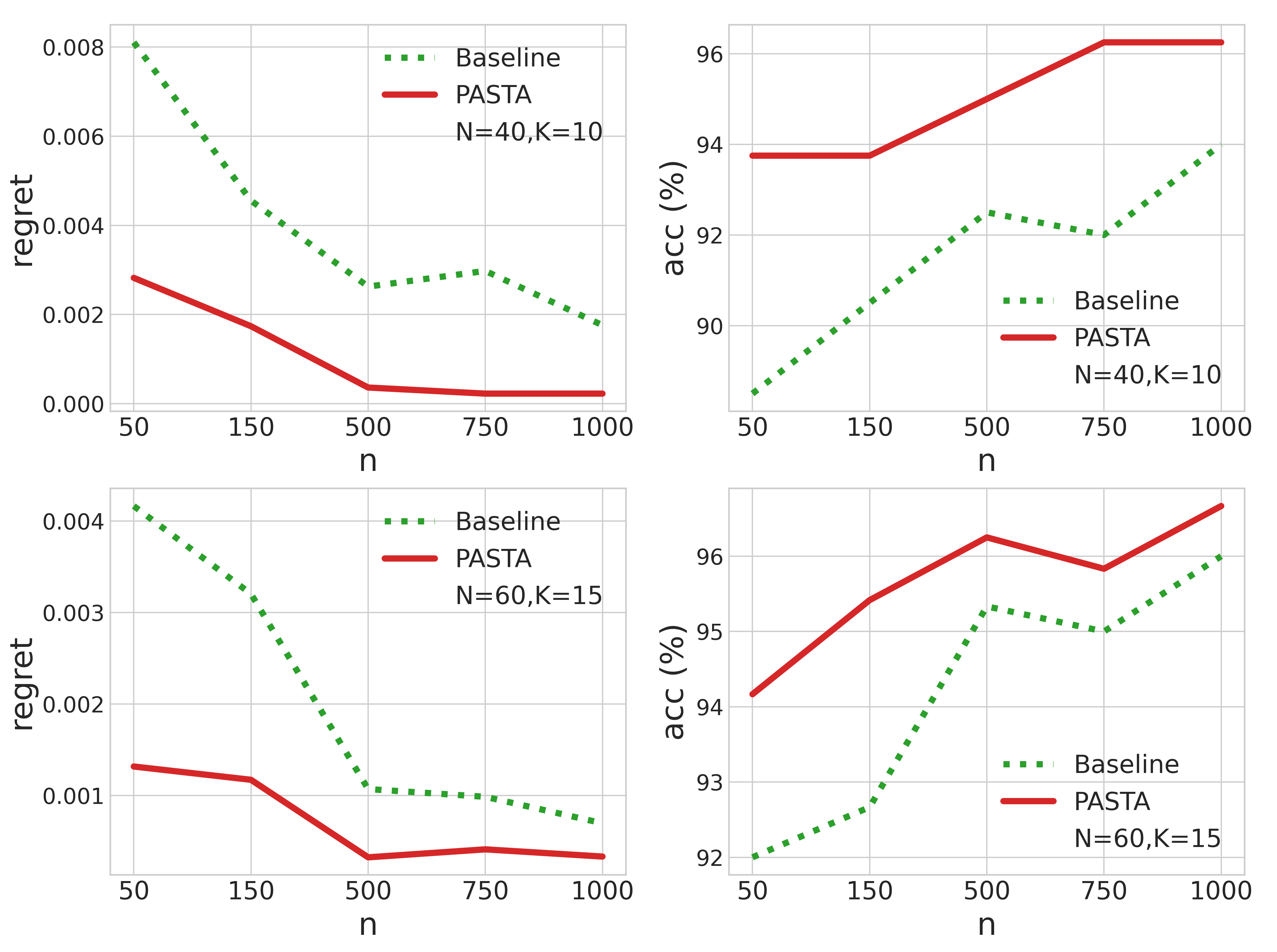}}
\caption{Performance comparison between PASTA and the baseline method with varying number of samples ($n$). On the left is the average regret (the lower the better) while the assortment accuracy (the higher the better) is on the right. 
}
\label{fig:exp_effect_n_40_60}
\end{center}
\vskip -0.2in
\end{figure}

\textbf{Effect of Probability of Sampling Optimal Assortment in Offline Data.} We set $N=40$, $K=8$, $d=16$, $n=150$, and let $p \in \{0.1, 0.3, 0.5, 0.7, 0.9\}$.
We also study the effect of $p$ in scenarios with an increased total number of products ($N=60$, $K=15$). As can be seen in Figure~\ref{fig:exp_effect_p}, the gain of pessimistic assortment optimization is consistent and robust for varying values of $p$.
\begin{figure}[h]
\vskip 0.2in
\begin{center}
\centerline{\includegraphics[width=0.7\columnwidth, scale = 0.5]{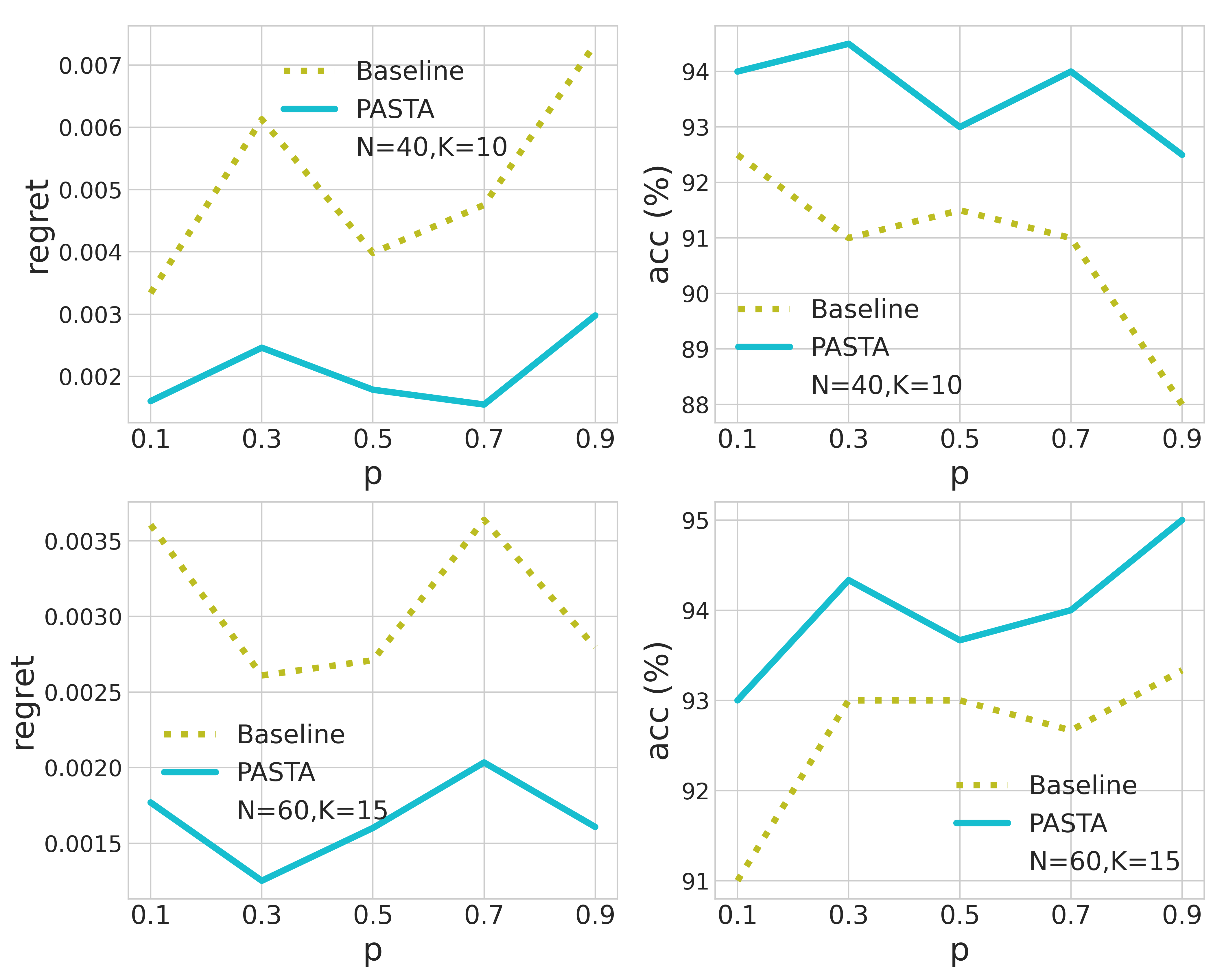}}
\caption{Comparison between PASTA and the baseline method with varying probability of the optimal assortment ($p$). Top row: $N=40$; bottom row: $N=60$.}
\label{fig:exp_effect_p}
\end{center}
\vskip -0.2in
\end{figure}

\textbf{Effect of Dimension of Features.} We set $N=20$, $K=5$, $p=0.9$, $n=150$, and let $d \in \{8,20,32,64,128\}$. In order to characterize the effect of dimension $d$, we generate $d$ elements of $\theta^*$ independently from $\text{Uniform}[-1,1]$. The results are presented in Figure~\ref{fig:exp_effect_d}. We observe that while both the regret of the baseline method and that of the pessimistic assortment optimization increase with increasing dimensions of features, the PASTA method maintains its performance gain as the dimension $d$ varies.  
\begin{figure}[H]
\vskip 0.2in
\begin{center}
\centerline{\includegraphics[width=0.7\columnwidth]{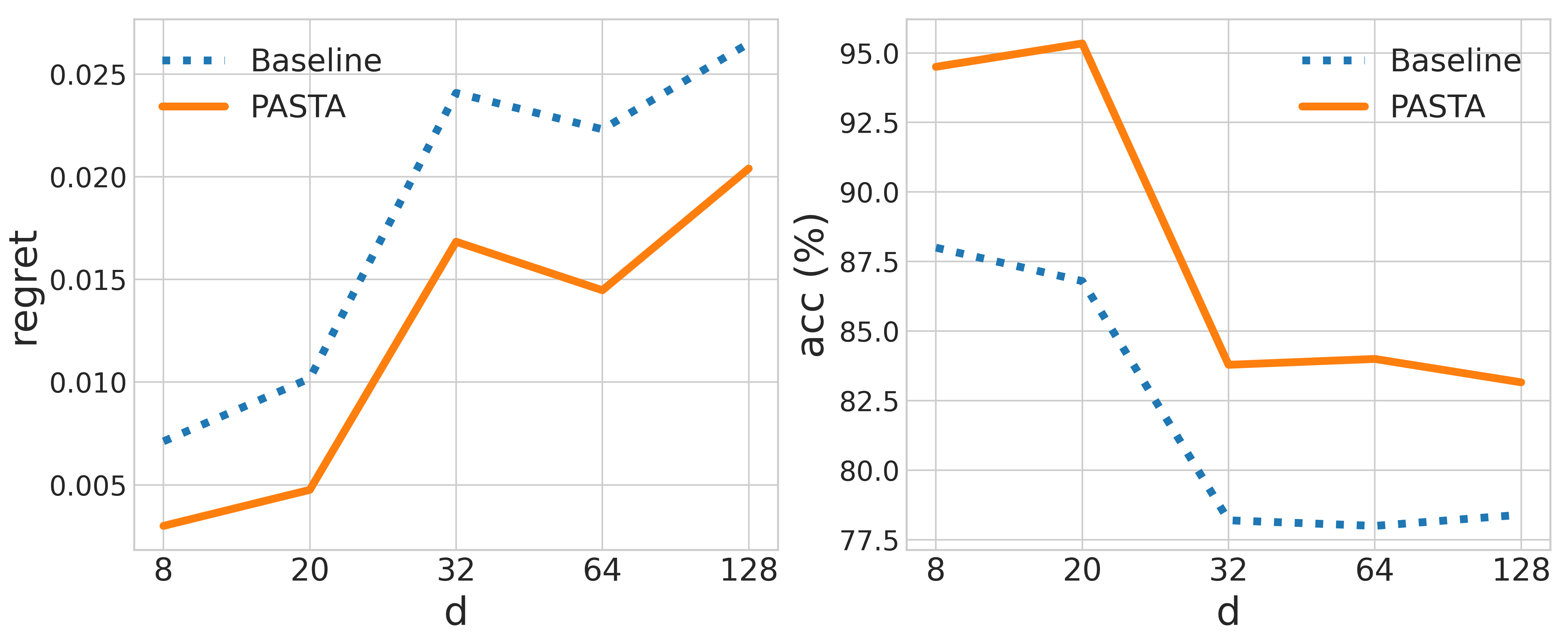}}
\caption{Comparison between PASTA and the baseline method with increasing dimensions of product features ($d$).}
\label{fig:exp_effect_d}
\end{center}
\vskip -0.2in
\end{figure}

\section{Conclusion}\label{sec:conclusion}
This work addresses the issue of insufficient data coverage in offline assortment optimization problems. This becomes more challenging as the number of choices grows quickly as a function of the number of items $N$.
We presented a framework of pessimistic assortment optimization and provided theoretical justifications for our approach. We then performed an in-depth study of the Multinomial Logit Model (MNL), and derived a finite-sample regret bound of pessimistic assortment optimization for this popular model. We presented an efficient algorithm to solve the pessimistic assortment optimization problem for MNL, and demonstrated significant improvements of our approach over the baseline method by extensive numerical studies.

\nocite{langley00}

\bibliography{icml_2023}
\bibliographystyle{icml2023}

\newpage
\appendix
\onecolumn
\label{sec:appendix}
\section{Proof of Theoretical Results}
Throughout the proofs, we use $\theta^*$ to denote the true parameter and $n$ to denote the number of samples. We use $\widehat L_n$ to denote the empirical negative log-likelihood function, i.e., $\widehat{L}_n(\theta)= -\frac{1}{n}\sum_{i=1}^{n}\log\pi_{A}(A_{i}|S_i;\theta)$ where $\{(S_i, A_{i}, R_{i})\}_{i=1}^n $ is the offline dataset. We use $L$ and $\widehat \theta_{\ML,n}$ to respectively denote the negative log-likelihood function , i.e., $L(\theta)=-\bbE[\log \pi_A(A|S;\theta)]$, and  the MLE of $\theta^*$, i.e.,
$\widehat{\theta}_{\ML,n} \in \argmin_{\theta \in \Theta} \left\{\widehat{L}_n(\theta)\right\}$. The confidence region $\Omega_n$ is defined as $\Omega_{n} \doteq \{ \theta \in \Theta: \widehat{L}_n(\theta) - \widehat{L}_n(\widehat{\theta}_{\ML,n}) \leq \alpha_{n}\}.$

For a general pair of random variables $(X,Y)$, assume that the conditional probability density function of $Y$ given $X$ is parameterically modeled by $p(y|x;\theta)$ for  parameter $\theta$.
For technical reasons, we will consider the following distances. 
\begin{definition}[Squared Hellinger Distance]
\begin{equation}
\label{def:h2}
h^2(p(\cdot|x;\theta_1),p(\cdot|x;\theta_2))=\frac{1}{2}\int\left( \sqrt{p_1(y|x;\theta_1)}-\sqrt{p_2(y|x;\theta_2)} \right)^2 dy.
\end{equation}
\end{definition}

\begin{definition}[Hellinger Distance]
\begin{equation}
h(p(\cdot|x;\theta_1),p(\cdot|x;\theta_2)) = \sqrt{h^2(p(\cdot|x;\theta_1),p(\cdot|x;\theta_2))}.
\end{equation}
\end{definition}

\begin{definition}[Generalized Squared Hellinger Distance]
 \begin{equation}\label{def:H2}
H^2(\theta_1,\theta_2) = \mathbb{E}_X\Big[h^2(p(\cdot|X;\theta_1),p(\cdot|X;\theta_2))\Big].
\end{equation}   
\end{definition}

\begin{definition}[Generalized Hellinger Distance]
\begin{equation}
    H(\theta_1,\theta_2) = \mathbb{E}_X\Big[\sqrt{h^2(p(\cdot|X;\theta_1),p(\cdot|X;\theta_2))}\Big].
\end{equation}
\end{definition}

In our theoretical results, we particularly consider  $p(y|x;\theta) = \pi_{A}(a|s;\theta)$ as the conditional density of $A$ given $S$ (hereafter denoted as $A|S$).

\label{sec:proof_lemma_general_pasta_theta}
\subsection{Proof of Lemma~\ref{lemma:general_pasta_theta}}
Under Assumption \ref{ass: technical assumptions}, for any $0 < \delta < 1$, with probability at least $1-\delta$, we have for any $\theta \in \Omega_{n} $,
\begin{align}
 \cV(s^{*};\theta^{*}) -  \cV(s^{*};\theta) \lesssim \: r_{s^{\ast}} C_{s^*}\sqrt{\alpha_n},
\end{align}
where $r_{s^{\ast}} \doteq \max_{j \in s^*}r_{j}$ is the largest possible revenue among all items in the optimal assortment.

\begin{proof}[Proof of Lemma~\ref{lemma:general_pasta_theta}]
For any $\theta$ such that $\widehat L_n(\theta) - \widehat L_n(\widehat \theta_{\ML,n}) \leq \alpha_n$, i.e., $\theta \in \Omega_n$, we have 
\begin{align*}
    &\cV(s^{*};\theta^{*})-\cV(s^*;\theta) \leq \bigg| \cV(s^*;\theta)-\cV(s^{*};\theta^{*}) \bigg| \\
    &\leq \bigg| \cV(s^*;\theta)-\cV(s^*;\widehat{\theta}_{\ML,n})+\cV(s^*;\widehat{\theta}_{\ML,n})-\cV(s^{*};\theta^{*}) \bigg|\\
    &\leq \bigg| \cV(s^*;\theta)-\cV(s^*;\widehat{\theta}_{\ML,n})\bigg| +\bigg|\cV(s^*;\widehat{\theta}_{\ML,n})-\cV(s^{*};\theta^{*}) \bigg| \\
    &\leq 2\max_{\theta \in \Omega_n}\bigg| \cV(s^*;\theta)-\cV(s^*;\widehat{\theta}_{\ML,n})\bigg| \qquad \qquad  \text{(By Assumption \ref{ass: technical assumptions}~(II) $\theta^* \in \Omega_n$)}.
\end{align*}
With Lemma~\ref{lemma:l1_distance_ub}, we have that for any $\theta \in \Theta$,
\begin{align*}
    &\bigg| \cV(s^*;\theta)-\cV(s^*;\widehat{\theta}_{\ML,n}) \bigg| \leq r_{s^{\ast}}C_{s^*}\mathbb{E}_{S}\bigg[||\pi_A(\cdot|S;\theta)-\pi_A(\cdot|S;\widehat{\theta}_{\ML,n})||_{1}\bigg]
\end{align*}
where $||\cdot||_1$ is the $\ell_1$-norm, $r_{s^{\ast}} \doteq \max_{j \in s^*}r_{j}$ is the largest possible revenue among all items and $C_{s^*}=1/\pi_S(s^*)$.

In Lemma~\ref{lemma:l1_H_ub}, we establish that 
\[ \mathbb{E}_{S}\bigg[||\pi_A(\cdot|S;\theta)-\pi_A(\cdot|S;\widehat{\theta}_{\ML,n})||_{1}\bigg] \leq 2\sqrt{2}\sqrt{H^2(\theta, \widehat{\theta}_{\ML,n})}, \]
where $H^2$ is the generalized squared Hellinger distance defined in \eqref{def:H2} with $p(y|x;\theta) = \pi_{A}(a|s;\theta)$ as the conditional density of $A|S$.

Combining the above two inequalities, we have that for any $\theta \in \Theta$,
\begin{align}\label{eqn:regret_ub_H2}
\bigg| \cV(s^*;\theta)-\cV(s^*;\widehat{\theta}_{\ML,n})\bigg| \lesssim r_{s^{\ast}}C_{s^*}\sqrt{H^2(\theta,\widehat{\theta}_{\ML,n})}.
\end{align}

In the following, we use the fact that $\log x \leq 2(\sqrt{x}-1)$ for any $x \ge 0$ to show that for any $s\in\mathbb{S}$ and any $\theta$:
\begin{align*}
    &-\int\pi_A(a|s;\theta^\ast)\log\frac{\pi_A(a|s;\theta)}{\pi_A(a|s;\theta^\ast)}da \\
&\geq -2\int \pi_A(a|s;\theta^*)\left(\sqrt{\frac{\pi_A(a|s;\theta)}{\pi_A(a|s;\theta^*)}-1}\right)da\\
&=\int\left(\pi_A(a|s;\theta^*)+\pi_A(a|s;\theta)-2\sqrt{\pi_A(a|s;\theta)\pi_A(a|s;\theta^*)}\right)da\\
&=\int\left(\sqrt{\pi_A(a|s;\theta^*)}+\sqrt{\pi_A(a|s;\theta)}\right)^2da\\
&\geq \int\left(\sqrt{\pi_A(a|s;\theta^*)}-\sqrt{\pi_A(a|s;\theta)}\right)^2da, 
\end{align*}
which implies that
\begin{equation}\label{eqn:h2_ub}
L(\theta) - L(\theta^*) \geq 2H^2(\theta;\theta^*).
\end{equation}

By Lemma~\ref{lemma:property_H}, we have that for any $\theta \in \Omega_n$,
\begin{equation}
\label{eqn:proof_H2_inequality}
H^2(\theta,\widehat{\theta}_{\ML,n}) \leq 2H^2(\theta^*, \theta)+2H^2(\theta^*, \widehat{\theta}_{\ML,n})
 \leq L(\widehat \theta_{\ML,n}) - L(\theta^*) + \{L(\theta) - L(\theta^*)\}.
\end{equation}

From Assumption~\ref{ass: technical assumptions}, 
we have that with probability at least $1-\delta$,
for any $\theta \in \Omega_n$,
\begin{align*}
    \left|L(\theta)-L(\theta^*)-\left(\widehat{L}_{n}(\theta)- \widehat{L}_{n}(\theta^*)\right)\right|\le \alpha_n.
\end{align*}
In other words, under Assumption~\ref{ass: technical assumptions}, with probability at least $1-\delta$ for any $\theta \in \Omega_n$, 
\begin{align}\label{eqn:proof_likelihood_alpha}
L(\theta)-L(\theta^*) \leq \big|\widehat{L}_{n}(\theta)- \widehat{L}_{n}(\theta^*)\big| + \alpha_n \leq 2\alpha_n.
\end{align}
Plugging Eq.~\eqref{eqn:proof_likelihood_alpha} into Eq.~\eqref{eqn:proof_H2_inequality}, we have that with probability at least $1-\delta$, for any $\theta \in \Omega_n$, 
\begin{equation}
H^2(\theta,\widehat{\theta}_{\ML,n}) \leq 4\alpha_n.
\end{equation}
Combining the above inequality and Eq.~\eqref{eqn:regret_ub_H2}, we have that, with probability at least $1-\delta$, $\big| \cV(s^*;\theta)-\cV(s^*;\widehat{\theta}_{\ML}))\big| \lesssim r_{s^{\ast}}C_{s^*}\sqrt{\alpha_n}$ for all $\theta \in \Omega_n$ . This concludes the proof. 
\end{proof}

\label{sec:proof_thm_empirical_likelihood_bound}
\subsection{Proof of Lemma~\ref{thm:empirical_likelihood_bound}}
Consider the MNL model \eqref{eqn:mnl_ass} with a compact set $\Theta$. Assume that $\theta^\ast \in \Theta$. For $0 < \delta < 1$,  with probability at least $1-\delta$, we have
\begin{align}\label{eq:coverage}
    \widehat{L}_{n}(\theta^*)-\widehat{L}_{n}(\widehat\theta_{\ML,n}) \lesssim \frac{C_{A}d}{n}\log\frac{\theta_{\max}}{\delta}.
\end{align}

\begin{proof}[Proof of Lemma~\ref{thm:empirical_likelihood_bound}]

Fix $0 < \delta < 1$.
Suppose $\alpha_n \eqsim {C_{A}d \over n}\log{2\theta_{\max} \over \delta}$.
Define an oracle confidence set as
\[ \widetilde{\Omega}_{n} \doteq \left\{ \theta \in \Theta: L(\theta) - L(\theta^\ast) \le \alpha_{n} \right\}. \]
In particular, $\theta^\ast \in \widetilde{\Omega}_{n}$. By Lemmas \ref{lemma:likelihood_H2} and \ref{lemma:mle_hellinger}, we also have with probability at least $1 - \delta/2$,
\[ L(\widehat{\theta}_{\ML,n}) - L(\theta^\ast) \le 2C_{A}H^{2}(\widehat{\theta}_{\ML,n}, \theta^\ast) \lesssim \alpha_n, \]
that is, $\widehat{\theta}_{\ML,n} \in \widetilde{\Omega}_n$.

Define $\cF_{n} \doteq \left\{\log{\pi_{A}(A|S;\theta^\ast) \over \pi_{A}(A|S;\theta)}:\theta \in \widetilde{\Omega}_{n}\right\}$. In particular, for any $f \in \cF_n$, we have $\|f\|_{\infty} \le 2\log C_{A}$.
Let $\|\bbP_{n} - \bbP\|_{\cF_n} \doteq \sup_{f \in \cF_n}|(\bbP_{n} - \bbP)(f)| $ be the envelope.
By Talagrand's inequality, with probability at least $1-\delta/2$, we have for any $f \in \cF_n$,
\begin{align}
    (\bbP_n - \bbP)(f) \lesssim \bbE\|\bbP_{n} - \bbP\|_{\cF_n} + \sqrt{\frac{1}{n}\left\{(\log C_{A})\bbE\|\bbP_{n} - \bbP\|_{\cF_n} + \sup_{f \in \cF_n}\bbE(f-\bbE f)^2\right\}\log\frac{2}{\delta}} +  {\log C_{A} \over n}\log\frac{2}{\delta}.
    \label{eq:talagrand}
\end{align}
For the variance term, we have
\[ \begin{aligned}
    & \sigma_{\cF_{n}}^2 \doteq \sup_{f \in \cF_n}\bbE(f-\bbE f)^2 \le \sup_{\theta \in \widetilde{\Omega}_n}\bbE\left(\log{\pi_{A}(A|S;\theta^*) \over \pi_{A}(A|S;\theta)}\right)^2 \\ 
    & \le \sup_{\theta \in \widetilde{\Omega}_n}\bbE\left\{ 2C_Ah^2\big( \pi_{A}(\cdot|S;\theta^*), \pi_{A}(\cdot|S;\theta) \big) \right\} && (\text{by Lemma \ref{lemma:log-density-ratio_variance}}) \\
    & = 2C_{A}\sup_{\theta \in \widetilde{\Omega}_n}H^2(\theta^*, \theta) \\
    & \le 2C_{A}\sup_{\theta \in \widetilde{\Omega}_n}[L(\theta) - L(\theta^*)] && (\text{by \eqref{eqn:proof_H2_inequality}}) \\
    & = 2C_A\alpha_n. && (\text{by definition of $\widetilde{\Omega}_n$})
\end{aligned} \]
For the expected envelope, our goal below is to apply \citet[Theorem 7.13]{sen2018gentle} (stated in Theorem \ref{thm:expected_sup}). Consider the covering number $\cN(\epsilon, \cF_{n}, L^2(\bbQ))$ for any given  $\epsilon > 0$ and finitely supported probability measure $\bbQ$. By Lemma \ref{lemma:lipschitz}, based on the MNL model \eqref{eqn:mnl_ass}, for some $\sfL < +\infty$, $\cF_n$ is a class of $\sfL$-Lipschitz functions with respect to the index space $(\Theta,\|\cdot\|_2)$. Then in terms of the bracketing number $\cN_{[]}$ and covering number $\cN$, for any $\epsilon > 0$ and probability measure $\bbQ$, we have
\[ \cN(\epsilon \sfL, \cF_{n}, L^2(\bbQ)) \le \cN_{[]}(2\epsilon \sfL, \cF_{n}, L^2(\bbQ)) \le \cN(\epsilon, \Theta, \|\cdot\|_2). \]
By $\Theta \subseteq \bbR^d$ and $\Theta$ is compact, we further have $\cN(\epsilon,\Theta,\|\cdot\|_2) \lesssim \left({1 \over \epsilon}\right)^{d} $. Therefore,
\[ \cN(\epsilon, \cF_n, L^2(\bbQ))\lesssim \left({\sfL \over \epsilon}\right)^{d}. \]
By Theorem \ref{thm:expected_sup}, we further have
\[ \bbE\|\bbP_n - \bbP\|_{\cF_n} \lesssim 
\sqrt{{d \over n}\sigma_{\cF_{n}}^2 \log{\sfL \over \sigma_{\cF_{n}}}} \vee \left\{{d \over n}\times 2\log(C_A)\log{\sfL \over \sigma_{\cF_n}}\right\}
\lesssim \sqrt{{C_{A}d \over n}\alpha_n} \eqsim {C_{A}d \over n}\sqrt{\log{2\theta_{\max} \over \delta}} \lesssim \alpha_n.\]
The Talagrand's inequality \eqref{eq:talagrand} becomes
\[ (\bbP_n - \bbP)(f) \lesssim \alpha_n. \]
In particular, in the case of $\widehat{\theta}_{\ML,n} \in \widetilde{\Omega}_n$ corresponding to $f(\widehat{\theta}_{\ML,n}) \in \cF_n $, we have
\[ \widehat{L}_n(\theta^\ast) - \widehat{L}_n(\widehat{\theta}_{\ML,n}) = -\bbP_n f(\widehat{\theta}_{\ML,n}) \lesssim -\bbP f(\widehat{\theta}_{\ML,n}) + \alpha_n = \underbrace{L(\theta^\ast) - L(\widehat{\theta}_{\ML,n})}_{\le 0} + \alpha_n \le \alpha_n. \]
This complete the proof.
\end{proof}

\begin{lemma}
    \label{lemma:lipschitz}
    Consider the MNL model: 
    \[ \pi_{A}(i|s;\theta) = {\exp(x_i^\top\theta)\over 1 + \sum_{j \in s}\exp(x_j^\top\theta)}; \quad \pi_{A}(0|s;\theta) = {1 \over 1 + \sum_{j \in s}\exp(x_j^\top\theta)}; \quad \forall i \in s, ~ s \in \bbS, ~ \theta \in \Theta. \]
    Let $\theta_{\max} \doteq \max_{\theta \in \Theta}\|\theta\|_2$, $x_{\max} \doteq \max_{j \in [N]}\|x_{j}\|_2$. If $\Theta$ is compact, that is, $\theta_{\max} < +\infty$, then the log-likelihood ratio $ \log{\pi_{A}(A|S;\theta^\ast) \over \pi_{A}(A|S;\theta)} $ is a uniformly Lipschitz function in $\theta \in \Theta$.
\end{lemma}

\begin{proof}[Proof of Lemma \ref{lemma:lipschitz}]
    \begin{align}
        \begin{aligned}
        & \left\|{\partial \over \partial \theta}\log{\pi_{A}(A|S;\theta^*) \over \pi_{A}(A|S;\theta)}\right\|_2 = {1 \over \pi_{A}(A|S;\theta)}\left\|{\partial \pi(A|S;\theta) \over \partial \theta}\right\|_2 \\
        & = [1-\pi_{A}(A|S;\theta)]\left\|  x_{A} + \sum_{j \in S}\exp(x_j^\top\theta)(x_{A} - x_{j}) \right\|_2 \\
        & \le x_{\max} + 2Nx_{\max}\exp(x_{\max}\theta_{\max}) \doteq \sfL < +\infty.
        \end{aligned}
        \label{eq:lipschitz}
    \end{align}
    That is, $\theta \mapsto \log{\pi_{A}(A|S;\theta^\ast) \over \pi_{A}(A|S;\theta)} $ is $\sfL$-Lipschitz.
\end{proof}

\begin{lemma}[Concentration of Parametric MLE in Hellinger Distance]
    \label{lemma:mle_hellinger}
    Consider the MNL model \eqref{eqn:mnl_ass}. For $0 < \delta < 1$, with probability at least $ 1-\delta $, we have
    \[ H^2(\widehat{\theta}_{\ML,n},\theta^\ast) \lesssim {d \over n}\log{\theta_{\max} \over \delta}. \]
\end{lemma}

\begin{proof}[Proof of Lemma \ref{lemma:mle_hellinger}]
    We follow from \citet[Corollary 2]{fu2022offline} as a special case, where our data are generated i.i.d. instead of being a general Markov chain.
\end{proof}

\label{sec:proof_thm_likelyhood_convergence}
\subsection{Proof of Lemma~\ref{thm:likelyhood_convergence}}
Consider the MNL model \eqref{eqn:mnl_ass}. Suppose conditions in Lemma \ref{thm:empirical_likelihood_bound} hold, and  $L(\theta)$ and $L_n(\theta)$ are uniformly and strongly convex. Let $\alpha_{n} \eqsim {C_{A}d \over n}\log{\theta_{\max} \over \delta}$. For $ 0 < \delta < 1 $, with probability at least $ 1 - \delta $, we have
\[ \sup_{\theta \in \Omega_n}\left|L(\theta)-L(\theta^*)-\left(\widehat{L}_{n}(\theta)- \widehat{L}_{n}(\theta^*)\right)\right|\le \alpha_n. \]
\begin{proof}[Proof of Lemma~\ref{thm:likelyhood_convergence}]
    Fix $0 < \delta < 1$. By the strong convexity assumption on $L(\theta)$ and $L_n(\theta)$, there exists a constant $\mu > 0$ such that for any $\theta \in \Theta$,
    \[ \mu\| \theta - \theta^* \|_2^2 \le L(\theta) - L(\theta^*); \quad \mu\| \theta - \widehat{\theta}_{\ML,n} \|_2^2 \le \widehat{L}_n(\theta) - \widehat{L}_n(\widehat{\theta}_{\ML,n}). \]
    By Lemma \ref{thm:empirical_likelihood_bound}, with probability at least $1 - \delta/2$, we have $\widehat{\theta}_{\ML,n} \in \widetilde{\Omega}_n$. Then for any $\theta \in \Omega_n$, we have
    \[ \begin{aligned}
        & \|\theta - \theta^*\|_2 \le \|\theta - \widehat{\theta}_{\ML,n}\|_2 + \|\widehat{\theta}_{\ML,n} - \theta^*\|_2 &&(\text{by triangular inequality}) \\
        & \le {1 \over \sqrt{\mu}}\sqrt{\widehat{L}_n(\theta) - \widehat{L}_n(\widehat{\theta}_{\ML,n})} + {1 \over \sqrt{\mu}}\sqrt{L(\widehat{\theta}_{\ML,n}) - L(\theta^*)} && (\text{by strong convexity}) \\
        & \lesssim \sqrt{\alpha_{n}} && (\text{by $\theta \in \Omega_n$ and $\widehat{\theta}_{\ML,n} \in \widetilde{\Omega}_n$ respectively}).
    \end{aligned} \]
    
    The above implies that with probability at least $1 - \delta/2$, we have $\Omega_n \subseteq \widebar{\Omega}_n$, where $ \widebar{\Omega}_{n} $ is a ball centered around $\theta^*$ with radius $\sqrt{\alpha}_n$:
    \[ \widebar{\Omega}_{n} \doteq \left\{ \theta \in \Theta: \|\theta - \theta^*\|_2 \le \sqrt{\alpha}_{n} \right\}. \]

    Define $\widebar{\cF}_{n} \doteq \left\{\log{\pi_{A}(A|S;\theta^\ast) \over \pi_{A}(A|S;\theta)}:\theta \in \widebar{\Omega}_{n}\right\}$. 
    Let $\|\bbP_{n} - \bbP\|_{\widebar{\cF}_n} \doteq \sup_{f \in \widebar{\cF}_n}|(\bbP_{n} - \bbP)(f)| $ be the envelop.
    By Talagrand's inequality, with probability at least $1-\delta/2$, we have for any $f \in \widebar{\cF}_n$,
    \begin{align}
        |(\bbP_n - \bbP)(f)| \lesssim \bbE\|\bbP_{n} - \bbP\|_{\widebar{\cF}_n} + \sqrt{\frac{1}{n}\left\{(\log C_{A})\bbE\|\bbP_{n} - \bbP\|_{\widebar{\cF}_n} + \sup_{f \in \widebar{\cF}_n}\bbE(f-\bbE f)^2\right\}\log\frac{2}{\delta}} +  {\log C_{A} \over n}\log\frac{2}{\delta}.
        \label{eq:talagrand_2}
    \end{align}
    For the variance term, we have
    \[ \begin{aligned}
        & \sigma_{\widebar{\cF}_{n}}^2 \doteq \sup_{f \in \widebar{\cF}_n}\bbE(f-\bbE f)^2 \le \sup_{\theta \in \widebar{\Omega}_n}\bbE\left(\log{\pi_{A}(A|S;\theta^*) \over \pi_{A}(A|S;\theta)}\right)^2 \\ 
        & \lesssim \sup_{\theta \in \widebar{\Omega}_n}\|\theta - \theta^*\|_2^2 && (\text{by Lipschitzness in Lemma \ref{lemma:lipschitz}}) \\
        & \le \alpha_n && (\text{by definition of $\widebar{\Omega}_n$}).
    \end{aligned} \]
    For the expected envelope, by Theorem \ref{thm:expected_sup}, we further have
    \[ \bbE\|\bbP_n - \bbP\|_{\widebar{\cF}_n} \lesssim 
    \sqrt{{d \over n}\sigma_{\widebar{\cF}_n}^2 \log{\sfL \over \sigma_{\widebar{\cF}_n}}} \vee \left\{{d \over n}\times 2\log(C_A)\log{\sfL \over \sigma_{\widebar{\cF}_n}}\right\}
    \lesssim \sqrt{{d \over n}\alpha_n} \lesssim \alpha_n.\]
    Therefore, the Talagrand's inequality \eqref{eq:talagrand_2} gives that for any $f \in \widebar{\cF}_n$
    \[ |(\bbP_n - \bbP)(f)| \lesssim \alpha_n. \]
    In other words, with probability at least $1 - \delta$,  $\Omega_n \subseteq \widebar{\Omega}_n$ corresponding to $\{f(\theta)\}_{\theta \in \Omega_n} \subseteq \widebar{\cF}_n$, we have
    \[ \sup_{\theta \in \Omega_n}\left|L(\theta)-L(\theta^*)-\left(\widehat{L}_{n}(\theta)- \widehat{L}_{n}(\theta^*)\right)\right| \le \sup_{f \in \widebar{\cF}_n}|(\bbP_n - \bbP)(f)| \lesssim \alpha_n. \]

\end{proof}

\section{Technical Lemmas}

\begin{lemma}\label{lemma:likelihood_H2}
Suppose $C_{A} > 2$ and $C_A \ge 1/\pi_A(i|s;\theta)$ for all $\theta \in \Theta$, $ s \in \bbS $ and $i \in s$. Then for any $\theta \in \Theta$:
\begin{equation}
 |L(\theta)-L(\theta^*)| \leq 2C_A H^2(\theta,\theta^*),
\end{equation}
\end{lemma}

\begin{proof}[Proof of Lemma~\ref{lemma:likelihood_H2}]
By definition,
\begin{align*}
    |L(\theta^*)-L(\theta)| &= \left|\mathbb{E}\left\{\mathbb{E}_{A}\left[\log\frac{\pi_A(A|S;\theta^*)}{\pi_A(A|S;\theta)}\middle|S\right]\right\}\right|.
\end{align*}
In particular, for a fixed $s \in [N]$, we have
\begin{align*}
&\mathbb{E}\left[\log\frac{\pi_A(A|S;\theta^*)}{\pi_A(A|S;\theta)}\middle|S=s\right] =\KL\Big(\pi_A(\cdot|s;\theta^*) \Big|\Big| \pi_A(\cdot|s;\theta)\Big)\\
&\leq\frac{\log(C_{A}-1)}{1-2/C_{A}}\Big\{1-\big[1-h^2(\pi_A(\cdot|s;\theta^*), \pi_A(\cdot|s;\theta))\big]^2\Big\} \\
& \quad (\text{by log-Sobolev inequality \citep[Theorem A.1]{diaconis1996logarithmic}}) \\
&= \frac{C_{A}\log(C_{A}-2+1)}{C_{A}-2}\{2h^{2} - (h^{2})^2\} ~ \Big(\text{with}~ h^{2} = h^2\big(\pi_A(\cdot|s;\theta^*), \pi_A(\cdot|s;\theta)\big)\Big) \\
&\le 2C_{A}h^2(\pi_A(\cdot|s;\theta^*), \pi_A(\cdot|s;\theta)).
\end{align*}
Therefore,
$$
|L(\theta^*)-L(\theta)| \le \bbE\left\{2C_{A}h^2(\pi_A(\cdot|S;\theta^*), \pi_A(\cdot|S;\theta))\right\} = 2C_{A}H^{2}(\theta^\ast,\theta).
$$
\end{proof}

\begin{lemma}\label{lemma:l1_distance_ub}
Let $C_{s^*} \doteq \frac{1}{\pi_S(s^*)}$ and $r_{s^{\ast}} \doteq \max_{j \in s^*}r_{j}$, then the following inequality holds for any $\theta_1, \theta_2 \in \Theta$:
\begin{equation*}
    \bigg| \cV(s^*;\theta_1)-\cV(s^*;\theta_2)\bigg| \leq r_{s^{\ast}}C_{s^*}\mathbb{E}_{S}\bigg[||\pi_A(\cdot|S;\theta_1)-\pi_A(\cdot|S;\theta_2)||_{1}\bigg]
\end{equation*}
where $||\cdot||_{1}$ denotes the $L^1$ norm. 
\end{lemma}
\begin{proof}[Proof of Lemma~\ref{lemma:l1_distance_ub}]
\begin{align}
    \label{eqn:proof_sample_estimation}
    &\bigg| \cV(s^*;\theta_1)-\cV(s^*;\theta_2)\bigg| 
    = \bigg|\mathbb{E}_{S}\Biggr[\frac{\mathbb{I}(S=s^*)}{\pi_S(S)}\sum_{i\in S}r_{i,S}\bigg(\pi_A(i|S;\theta_1)-\pi_A(i|S;\theta_2)\bigg)\Biggr]\bigg| \\
    & \leq \mathbb{E}_{S}\Biggr[\bigg|\frac{\mathbb{I}(S=s^*)}{\pi_S(S)}\sum_{i\in S}r_{i,S}\bigg(\pi_A(i|S;\theta_1)-\pi_A(i|S;\theta_2)\bigg)\bigg|\Biggr]\\
    & \leq\bigg|\bigg| \frac{\mathbb{I}(S=s^*)}{\pi_S(S)} \bigg|\bigg|_{\infty}\mathbb{E}_{S}\Biggr[\mathbb{I}(S=s^*)\bigg|\sum_{i\in S}r_{i,S}\bigg(\pi_A(i|S;\theta_1)-\pi_A(i|S;\theta_2)\bigg)\bigg|\Biggr] \label{eqn:cauchy_ineqn}\\
    \label{eqn:infty_norm}
    &\leq C_{s^*}\mathbb{E}_{S}\Biggr[\mathbb{I}(S=s^*)\sum_{i\in S}\bigg|r_{i,S}\bigg(\pi_A(i|S;\theta_1)-\pi_A(i|S;\theta_2)\bigg)\bigg|\Biggr] \\
    &\leq r_{s^{\ast}}C_{s^*}\mathbb{E}_{S}\Biggr[\sum_{i\in S}\bigg|\bigg(\pi_A(i|S;\theta_1)-\pi_A(i|S;\theta_2)\bigg)\bigg|\Biggr]\\
    &=r_{s^{\ast}}C_{s^*}\mathbb{E}_{S}\bigg[||\pi_A(\cdot|S;\theta_1)-\pi_A(\cdot|S;\theta_2)||_{1}\bigg],
\end{align}
where Eq.~\eqref{eqn:proof_sample_estimation} comes from the sample-based estimation of $\bbE[R(s)]$ (Eq.~\eqref{eq:IPW}), Eq.~\eqref{eqn:cauchy_ineqn} comes from the H\"older's inequality, Eq.~\eqref{eqn:infty_norm} comes from the fact that $\bigg|\bigg| \frac{\mathbb{I}(S=s^*)}{\pi_S(S)} \bigg|\bigg|_{\infty} = C_{s^*}$ because $\frac{\mathbb{I}(S=s^*)}{\pi_S(S)}$ has the value zero everywhere except at $s=s^*$. The last equality follows from the definition of $L^1$ norm. 
\end{proof}

\begin{lemma}\label{lemma:l1_H_ub}
For any $\theta_1, \theta_2 \in \Theta$,
$$
\mathbb{E}_{S}\bigg[||\pi_A(\cdot|S;\theta_1)-\pi_A(\cdot|S;\theta_2)||_{1}\bigg] \leq 2\sqrt{2}\sqrt{H^2(\theta_1,\theta_2)}.
$$
\end{lemma}
\begin{proof}[Proof of Lemma~\ref{lemma:l1_H_ub}]
We first use the facts that (1) $L^1=\frac{1}{2}\text{TV}$ where TV is the total variation distance and (2) $\text{TV} \leq \sqrt{2}h$~\cite{tv_hellinger} where $h$ is the Hellinger distance to have that for any $s \in \bbS$,
\begin{equation}\label{eqn:l1_hellinger_ineqn}
||\pi_A(\cdot|s;\theta_1)-\pi_A(\cdot|s;\theta_2)||_{1} \leq 2\sqrt{2}h\big(\pi_A(\cdot|s;\theta_1),\pi_A(\cdot|s;\theta_2)\big).
\end{equation}
From \eqref{eqn:l1_hellinger_ineqn}, we have
\begin{align*}
    ||\pi_A(\cdot|s;\theta_1)-\pi_A(\cdot|s;\theta_2)||_{1} &\leq 2\sqrt{2}h\big(\pi_A(\cdot|s;\theta_1), \pi_A(\cdot|s;\theta_2)\big),\\
    ||\pi_A(\cdot|s;\theta_1)-\pi_A(\cdot|s;\theta_2)||_{1}^2 &\leq 8h^2\big(\pi_A(\cdot|s;\theta_1), \pi_A(\cdot|s;\theta_2)\big).
\end{align*}
Taking expectation with respect to $S$ on both sides, we have 
\begin{align*}
    \mathbb{E}_{S}\bigg[||\pi_A(\cdot|S;\theta_1)-\pi_A(\cdot|S;\theta_2)||_{1}^2\bigg] &\leq 8\mathbb{E}_{S}\bigg[h^2(\pi_A(\cdot|S;\theta_1), \pi_A(\cdot|S;\theta_2))\bigg], \\
    \mathbb{E}_{S}\bigg[||\pi_A(\cdot|S;\theta_1)-\pi_A(\cdot|S;\theta_2)||_{1}^2\bigg] &\leq 8H^2(\theta_1,\theta_2).
\end{align*}
By the Jensen's inequality, we have 
\begin{align*}
    \bigg[\mathbb{E}_{S}||\pi_A(\cdot|S;\theta_1)-\pi_A(\cdot|S;\theta_2)||_{1}\bigg]^2 \leq
    \mathbb{E}_{S}\bigg[||\pi_A(\cdot|S;\theta_1)-\pi_A(\cdot|S;\theta_2)||_{1}^2\bigg].
\end{align*}
This implies that 
\begin{align*}
    \mathbb{E}_{S}\bigg[||\pi_A(\cdot|S;\theta_1)-\pi_A(\cdot|S;\theta_2)||_{1}\bigg] \leq 2\sqrt{2}\sqrt{H^2(\theta_1,\theta_2)}.
\end{align*}
\end{proof}

\begin{lemma}[Properties of $H$ and $H^2$]\label{lemma:property_H}
For any $\theta_1, \theta_2, \theta_3 \in \Theta$, the following inequalities hold:
\begin{eqnarray*}
    H(\theta_1,\theta_2) &\leq& H(\theta_1,\theta_3) + H(\theta_2,p_3); \\
    \big(H(\theta_1,\theta_2)\big)^2 &\leq& H^2(\theta_1,\theta_2) \leq H(\theta_1,\theta_2); \\
     H^2(\theta_1,\theta_2) &\le& 2H^2(\theta_1,\theta_3) + 2H^2(\theta_2,\theta_3).
\end{eqnarray*}
\end{lemma}

\begin{proof}[Proof of Lemma~\ref{lemma:property_H}]
For ease of notation, for $i=1,2,3$, we use $p_i$ to denote $\pi_A$ parametrized by $\theta_i$, i.e., $p_i(a|s) = \pi_A(a|s;\theta_i)$.

(1) Notice that for any $s \in \mathbb{S}$, $\sqrt{h^2(p_1(\cdot|s),p_2(\cdot|s))}$ is just the regular Hellinger distance that satisfies the triangular inequality. Hence we have
\begin{align}\label{eqn:proof_h2_tri}
    \sqrt{h^2(p_1(\cdot|s),p_2(\cdot|s))} &\leq \sqrt{h^2(p_1(\cdot|s),p_3(\cdot|s))} + \sqrt{h^2(p_2(\cdot|s),p_3(\cdot|s))}. 
\end{align}
Take expectation of both side of Eq.~\eqref{eqn:proof_h2_tri} with respect to $S$, we have 
\begin{align*}
\mathbb{E}_{S}\Big[\sqrt{h^2(p_1(\cdot|S),p_2(\cdot|S))}\Big] &\leq \mathbb{E}_{S}\Big[\sqrt{h^2(p_1(\cdot|S),p_3(\cdot|S))}\Big]  + \mathbb{E}_{S}\Big[\sqrt{h^2(p_2(\cdot|S),p_3(\cdot|S))}\Big].
\end{align*}
By the definition of $H$, this means that
\begin{align*}
H(\theta_1,\theta_2) &\leq H(\theta_1,\theta_3) + H(\theta_2,\theta_3).
\end{align*}

(2) For the first inequality, by applying the Jensen's inequality, we have
\begin{equation}
\left(\mathbb{E}\bigg[\sqrt{h^2(p_1(\cdot|S),p_2(\cdot|S))}\bigg]\right)^2 \leq \mathbb{E}\bigg[\left(\sqrt{h^2(p_1(\cdot|S),p_2(\cdot|S))}\right)^2\bigg].
\end{equation} 
Then the inequality follows.

For the second inequality, we have 
\begin{align}
    H(\theta_1,\theta_2) - H^2(\theta_1,\theta_2) =&\mathbb{E}_{S}\Big[h(p_1(\cdot|S),p_2(\cdot|S))-h^2(p_1(\cdot|S),p_2(\cdot|S)) \Big] \\
    =&\mathbb{E}_{S}\Big[\Big(1-h(p_1(\cdot|S),p_2(\cdot|S))\Big)h(p_1(\cdot|S),p_2(\cdot|S)) \Big]. \label{eqn:proof_tri}
\end{align}
Note that for any $s$, $\Big(1-h(p_1(\cdot|s),p_2(\cdot|s))\Big)$ is a non-negative function because the Hellinger distance is no larger than $1$, and $h(p_1(\cdot|s),p_2(\cdot|s))$ is also a positive function because it is a metric. We have Eq.~\eqref{eqn:proof_tri} as the expectation of non-negative functions, and thus we have
$$
\mathbb{E}_{S}\Big[\Big(1-h(p_1(\cdot|S),p_2(\cdot|S))\Big)h(p_1(\cdot|S),p_2(\cdot|S)) \Big] \geq 0.
$$
Then the inequality follows. 

(3) Notice that for any $a,b,c$, we have $(a-b)^2 \leq 2(a-c)^2+2(b-c)^2$. With this fact, we have that for any $s$,
\begin{align}
    h^2(p_1(\cdot|s),p_2(\cdot|s))&=\frac{1}{2}\int\left( \sqrt{p_1(a|s)}-\sqrt{p_2(a|s)} \right)^2 da \\
    &\leq \frac{1}{2}\int 2\left( \sqrt{p_1(a|s)}-\sqrt{p_3(a|s)} \right)^2 + 2\left( \sqrt{p_2(a|s)}-\sqrt{p_3(a|s)} \right)^2 da \\
    &= 2\cdot \frac{1}{2}\int\left( \sqrt{p_1(a|s)}-\sqrt{p_3(a|s)} \right)^2 da + 2\cdot \frac{1}{2}\int\left( \sqrt{p_2(a|s)}-\sqrt{p_3(a|s)} \right)^2 da\\
    & = 2 h^2(p_1(\cdot|s),p_3(\cdot|s)) + 2 h^2(p_2(\cdot|s),p_3(\cdot|s)).
\end{align}
This implies that 
\begin{equation}
     H^2(\theta_1,\theta_2) \le 2H^2(\theta_1,\theta_3) + 2H^2(\theta_2,\theta_3).
\end{equation}
\end{proof}

\begin{lemma}[Log-Density Ratio Variance Bound]
    \label{lemma:log-density-ratio_variance}
    Suppose $X \sim p$ is an $\bbR$-valued random variable with probability density function $p$, and $p_1, p_2$ are two other probability density functions for $X$ such that $p_1$ and $p_2$ are uniformly bounded from below by $C^{-1}$ on the support of $p$. Then we have
    \[ \bbE_{X \sim p}\left(\log{p_1(X) \over p_2(X)}\right)^2 \le 2Ch^2(p_1, p_2), \]
    where $h^2$ is the squared Hellinger distance in \eqref{def:h2}.
\end{lemma}

\begin{proof}[Proof of Lemma \ref{lemma:log-density-ratio_variance}]
    By $\log(x) \le 2(\sqrt{x} - 1)$ for any $x \ge 0$, we have
    \[ \begin{aligned}
        & \bbE_{X \sim p}\left(\log{p_1(X) \over p_2(X)}\right)^2 = \int \left( \log{p_1(X) \over p_2(X)} \right)^2 p(x)\rd x \\
        & \le 4 \int \max\left\{ \left( \sqrt{p_1(x) \over p_2(x)}-1 \right)^2, \left( \sqrt{p_2(x) \over p_1(x)}-1 \right)^2 \right\} p(x) \rd x \\
        & = 4 \int \max\left\{ {1 \over p_2(x)}\left(\sqrt{p_1(x)}-\sqrt{p_2(x)} \right)^2, {1 \over p_1(x)}\left( \sqrt{p_2(x)}-\sqrt{p_1(x)} \right)^2 \right\} p(x) \rd x \\
        & \le 4C \int \left( \sqrt{p_1(x)} - \sqrt{p_2(x)} \right)^2\rd x \\
        & = 2Ch^2(p_1, p_2).
    \end{aligned} \]
\end{proof}

\begin{theorem}[{\citet[Theorem 7.13]{sen2018gentle}}] \label{thm:expected_sup}
    Let $\cF$ be a measurable function class, such that $\sup_{f \in \cF}\|f\|_{\infty} \le f_{\max} $ for some constant $f_{\max} < +\infty$. Assume that for $A \ge ef_{\max}$, $d \ge 2$, $0 \le \epsilon \le f_{\max}$, and every finitely supported probability measure $\bbQ$, we have the covering number~\cite{sen2018gentle} as:
    \begin{align}
        \cN(\epsilon, \cF, L^2(\bbQ)) \lesssim \left({A \over \epsilon}\right)^{d}.
        \label{eq:expected_sup}
    \end{align}
    Let $\sigma_{\cF}^2 \doteq \sup_{f \in \cF}\bbE(f - \bbE f)^2$.
    Then we have
    \[ \bbE\|\bbP_n - \bbP\|_{\cF} \lesssim \sqrt{{d \over n}\sigma_{\cF}^2\log{A \over \sigma_{\cF}}} \vee \left\{{d \over n}f_{\max}\log{A \over \sigma_{\cF}}\right\}. \]
\end{theorem}

\begin{proof}[Proof of Theorem \ref{thm:expected_sup}] 
    In this proof, we denote $X$ as the underlying random variable, $\{ X_i \}_{i=1}^n$ are $n$ i.i.d. copies of $X$, and for any $f \in \cF$, $\bbP_n(f) \doteq {1 \over n}\sum_{i=1}^n f(X_i)$, $\bbP(f) \doteq \bbE[f(X)]$.
    Without loss of generality, assume that $0 \in \cF$, and for any $f \in \cF$, $\bbP(f) = 0$.
    Let $\{\epsilon_{i}\}_{i = 1}^n$ be i.i.d. Rademacher random variables that are independent of $\{X_i\}_{i=1}^n$. By symmetrization, we have
    \begin{align}
        \bbE\|\bbP_n - \bbP\|_\cF \le 2\bbE\sup_{f \in \cF}\left|{1 \over n}\sum_{i=1}^n \epsilon_i f(X_i) \right|.
        \label{eq:expected_sup_1}
    \end{align}
    Conditional on $\{X_i\}_{i=1}^n$, by Dudley's entropy bound, we have
    \begin{align}
        \bbE_{\epsilon}\sup_{f \in \cF}\left|\sum_{i=1}^n \epsilon_i {f(X_i) \over \sqrt{n}} \right| \le \int_{0}^{\sigma_{\cF,n}}\sqrt{\log\cN(u, \cF, L^2(\bbP_n))}\rd u,
        \label{eq:expected_sup_2}
    \end{align}
    where we consider the $L^2(\bbP_n)$ as the metric on $\cF$, that is, for any $f \in \cF$, $\|f\|_{L^2(\bbP_n)}^2 = {1 \over n}\sum_{i=1}^nf(X_i)^2 $. We also denote $\sigma_{\cF,n}^2 \doteq \sup_{f \in \cF}\|f\|_{L^2(\bbP_n)}^2 $, and $\bbE_{\epsilon}$ to emphasize that the expectation is taken with respect to the Rademacher random variables $\{\epsilon_{i}\}_{i=1}^n $ but holding $\{X_i\}_{i=1}^n$ as fixed. By \eqref{eq:expected_sup} with $\bbQ$ chosen as $\bbP_n$, we have
    \begin{align}
        \begin{aligned}
        & \eqref{eq:expected_sup_2} \lesssim \sqrt{d}\int_{0}^{\sigma_{\cF,n}}\sqrt{\log{A \over \delta}} \rd \delta \\
        & \le 2\sqrt{d}\sigma_{\cF,n}\sqrt{\log{A \over \sigma_{\cF,n}}} && (\text{by Lemma \ref{lem:int}}).
        \end{aligned}
        \label{eq:expected_sup_3}
    \end{align}
    In particular, $\log(A/\sigma_{\cF,n}) \ge \log(A/f_{\max}) \ge 1 $ by assumption, which satisfies the condition for Lemma \ref{lem:int}.
    Combining \eqref{eq:expected_sup_1}, \eqref{eq:expected_sup_2} and \eqref{eq:expected_sup_3}, we have
    \begin{align}
        \begin{aligned}
            & \bbE\|\bbP_n - \bbP\|_\cF \lesssim \sqrt{d \over n} \times \bbE\sqrt{\sigma_{\cF,n}^2\log{A \over \sigma_{\cF,n}}} \\
            & \le \sqrt{d \over n}\times\sqrt{{1 \over 2}\bbE(\sigma_{\cF,n}^2)\log{A^2\over \bbE(\sigma_{\cF,n}^2)}} && \left(\text{by the concavity of $u \mapsto \sqrt{u\log{A^2 \over u}}$}\right),
        \end{aligned}
        \label{eq:expected_sup_4}
    \end{align}
    where $\bbE$ takes expectation with respect to $\{ X_{i} \}_{i=1}^n $. Notice that
    \[ \bbE(\sigma_{\cF,n}^2) = \bbE\sup_{f \in \cF}\bbP_n(f^2) \le \bbE\sup_{f \in \cF}\Big\{|(\bbP_n-\bbP)(f^2)| + \bbP(f^2)\Big\} \le \bbE\|\bbP_n - \bbP\|_{\cF^2} + \sigma_{\cF}^2, \]
    where we define $\cF^2 \doteq \{ f^2: f \in \cF\} $. We aim to apply \eqref{eq:expected_sup_1}, \eqref{eq:expected_sup_2}, \eqref{eq:expected_sup_3}, \eqref{eq:expected_sup_4} to $\cF^2$. Notice that $\sup_{f^2 \in \cF^2}\|f^2\|_\infty \le f_{\max}^2$, $\sigma_{\cF^2,n}^2 \doteq \sup_{f^2 \in \cF^2}\|f^2\|_{L^2(\bbP_n)}^2 \le f_{\max}^2\sup_{f \in \cF}\|f\|_{L^2(\bbP_n)}^2 = f_{\max}^2 \sigma_{\cF,n}^2$. By $\|f^2 - g^2\|_{L^2(\bbP_n)} = \sqrt{\bbP_{n}[(f+g)^2(f-g)^2]} \le 2f_{\max}\sqrt{\bbP_{n}(f-g)^2} = 2f_{\max}\|f-g\|_{L^2(\bbP_n)}^2 $ for any $f, g \in \cF$, we further have
    \[ \cN(2f_{\max}\epsilon,\cF^2,L^2(\bbP_n)) \le \cN(\epsilon,\cF,L^2(\bbP_n)) \lesssim \left({A \over \epsilon}\right)^d. \]
    Therefore,
    applying \eqref{eq:expected_sup_1}, \eqref{eq:expected_sup_2}, \eqref{eq:expected_sup_3} and \eqref{eq:expected_sup_4} to $\cF^2$, we have
    \[ \bbE\|\bbP_n - \bbP\|_{\cF^2} \lesssim \sqrt{d \over n}\times \sqrt{{1 \over 2}f_{\max}^2\bbE(\sigma_{\cF,n}^2)\log{4A^2 \over \bbE(\sigma_{\cF,n}^2)}}. \]
    Define $B \doteq \sqrt{{1 \over 2}\bbE(\sigma_{\cF,n}^2)\log{A^2 \over \bbE(\sigma_{\cF,n}^2)}}$. Then we have
    \[ \bbE(\sigma_{\cF,n}^2) - \sigma_{\cF}^2 \lesssim \sqrt{d \over n} f_{\max}B. \]
    By $u \mapsto u\log{A^2 \over u}$ is non-decreasing on $u \in (0,A^2/e]$ and non-increasing on $u \in [A^2/e,+\infty)$, we have \[ \begin{aligned}
        & B^2 = {1 \over 2}\bbE(\sigma_{\cF,n}^2)\log{A^2 \over \bbE(\sigma_{\cF,n}^2)} \lesssim {1 \over 2}\left\{ \left(\sigma_{\cF}^2 + \sqrt{d \over n}f_{\max}B\right) \wedge {A^2 \over e}\right\}\log{A^2 \over \left(\sigma_{\cF}^2 + \sqrt{d \over n}f_{\max}B\right)\wedge {A^2 \over e}}.
    \end{aligned} \]
    In particular, $B \le \sqrt{{A^2 \over 2e}}$, $\sigma_{\cF}^2 \le f_{\max}^2 \le A^2/e^2 < A^2/e $. Then the cap $A^2/e$ is inactive as $d/n \to 0$ asymptotically. Therefore,
    \[ B^2 \lesssim \left(\sigma_{\cF}^2 + \sqrt{d \over n}f_{\max}B\right)\log{A \over \sigma_{\cF}}. \]
    In particular, $B$ is bounded by both roots of the corresponding quadratic equation:
    \[ B \lesssim {1 \over 2}\left\{\sqrt{d \over n}f_{\max}\log{A \over \sigma_{\cF}} + \sqrt{{d \over n}f_{\max}^2\left(\log{A \over \sigma_{\cF}}\right)^2 + 4\sigma_{\cF}^2\log{A \over \sigma_{\cF}}}\right\} \lesssim \left\{\sqrt{d \over n}f_{\max}\log{A \over \sigma_{\cF}}\right\} \vee \sqrt{\sigma_{\cF}^2\log{A \over \sigma_{\cF}}}. \]
    Combined with \eqref{eq:expected_sup_4}, we further have
    \[ \bbE\|\bbP_n - \bbP\|_{\cF} \lesssim \sqrt{d \over n}B \lesssim \sqrt{{d \over n}\sigma_{\cF}^2\log{A \over \sigma_{\cF}}} \vee \left\{{d \over n}f_{\max}\log{A \over \sigma_{\cF}}\right\}. \]
\end{proof}

\begin{lemma}
    \label{lem:int}
    Suppose $a,A > 0$ such that $\log(A/a) \ge 1$. Then we have
    \[ \int_{0}^{a}\sqrt{\log{A \over u}}\rd u \le 2a\sqrt{\log{A \over a}}. \]
\end{lemma}

\begin{proof}[Proof of Lemma \ref{lem:int}]
    Define
    \[ f(a) \doteq \begin{cases}
        2a\sqrt{\log{A \over a}} - \int_{0}^{a}\sqrt{\log{A \over u}}\rd u, & a > 0; \\
        0, & a = 0.
    \end{cases} \]
    Then $f$ is continuous at $0$. Moreover, for $a > 0$, we have
    \[ f'(a) = \sqrt{\log{A \over a}} - {1 \over \sqrt{\log{A \over a}}}, \]
    which is nonnegative if $\log(A/a) \ge 1$. 
    As $a \to 0^+$, we further have
    \[ \begin{aligned}
        & {f(a) \over a} = 2\sqrt{\log{A \over a}} - {1 \over a}\int_{0}^a\sqrt{\log{A \over u}}\rd u \\
        & = \sqrt{\log{A \over a}} - {1 \over 2a}\int_{0}^a{1 \over \sqrt{\log{A \over u}}}\rd u &&(\text{by integration-by-part}) \\
        & \ge \sqrt{\log{A \over a}} - {1 \over 2\sqrt{\log{A \over a}}}; \\
        & \liminf_{a \to 0^+}{f(a) \over a} \ge +\infty.
    \end{aligned}  \]
    Therefore, for any $a \ge 0$, we have $f(a) \ge  0$, which concludes the proof.
\end{proof}

\end{document}